\documentclass[10pt]{article}
\usepackage{amsmath,amssymb,amsthm,hyperref,multirow}
\usepackage{enumitem}
\providecommand{\keywords}[1]{\textbf{\textit{Keywords---}} #1}

\DeclareMathOperator*{\argmin}{arg min}

\newcommand{\N}{\mathbb{N}} 
\newcommand{\R}{\mathbb{R}}

\newsavebox{\measurebox}
\usepackage[numbers,sort&compress]{natbib}
\usepackage{authblk}
\usepackage{algorithm}
\usepackage{algpseudocode}
\usepackage[table]{xcolor}
\usepackage{graphicx}
\usepackage{epstopdf}
\usepackage{subfig}
\usepackage{dsfont, physics,stmaryrd,tabu}
\usepackage{enumitem}
\usepackage{mathtools}
\DeclarePairedDelimiter\ceil{\lceil}{\rceil}
\DeclarePairedDelimiter\floor{\lfloor}{\rfloor}

\theoremstyle{plain}
\newtheorem{theorem}{Theorem}[section]
\newtheorem*{theorem*}{Theorem}
\newtheorem{lemma}[theorem]{Lemma}
\newtheorem{proposition}[theorem]{Proposition}

\theoremstyle{definition}

\newtheorem{definition}[theorem]{Definition}
\newtheorem{notation}[theorem]{Notation}
\theoremstyle{remark}

\newtheorem{remark}[theorem]{Remark}

\numberwithin{equation}{section}
\numberwithin{algorithm}{section}
\numberwithin{figure}{section}
\numberwithin{table}{section}
\numberwithin{theorem}{section}

\newcommand{\relu}{{\rm ReLU}}
\newcommand{\dist}{{\rm dist}}
\newcommand{\proj}{{\rm proj}}
\newcommand{\prox}{{\rm prox}}
\newcommand{\vect}{{\rm vec}}

\usepackage{array}
\newcolumntype{P}[1]{>{\centering\arraybackslash}p{#1}}
\title{Forward Stability of ResNet and Its Variants}
\author{Linan Zhang}
\author{Hayden Schaeffer}
\affil{Department of Mathematical Sciences, Carnegie Mellon University, Pittsburgh, PA 15213. (\text{linanz@andrew.cmu.edu}, { }\text{schaeffer@cmu.edu})}
\date{November, 2018}

\begin{document}

\maketitle

\begin{abstract}

The residual neural network (ResNet) is a popular deep network architecture which has the ability to obtain high-accuracy results on several image processing problems. In order to analyze the behavior and structure of ResNet, recent work has been on establishing connections between ResNets and continuous-time optimal control problems. In this work, we show that the post-activation ResNet is related to an optimal control problem with differential inclusions, and  provide continuous-time stability results for the differential inclusion associated with ResNet. Motivated by the stability conditions, we show that alterations of either the architecture or the optimization problem can generate variants of ResNet which improve the theoretical stability bounds. In addition, we establish stability bounds for the full (discrete) network associated with two variants of ResNet, in particular, bounds on the growth of the features and a measure of the sensitivity of the features with respect to perturbations. These results also help to show the relationship between the depth, regularization, and stability of the feature space.  Computational experiments on the proposed variants show that the accuracy of ResNet is preserved and that the accuracy seems to be monotone with respect to the depth and various corruptions.

\end{abstract}

\keywords{Deep Feedforward Neural Networks, Residual Neural Networks, Stability, Differential Inclusions, Optimal Control Problems.}

\section{Introduction}

Deep neural networks (DNNs) have been successful in several challenging data processing tasks, including but not limited to: image classification, segmentation, speech recognition, and text analysis.  The first convolutional neural network (CNN), which was used in the recognition of digits and characters, was the famous LeNet \cite{lecun1989backpropagation}. The LeNet architecture included two convolution layers and two fully connected layers. Part of the success of CNNs is their ability to capture spatially local and hierarchal features from images. In \cite{krizhevsky2012imagenet}, the authors proposed a deeper CNN architecture, called AlexNet, which achieved record-breaking accuracy on the ILSVRC-2010 classification task \cite{ILSVRC15}. In addition to the increased depth ({\it i.e.} the number of layers), AlexNet also used rectified linear unit ($\relu$) as its activation function and overlapping max pooling to down-sample the features between layers. Over the past few years, the most popular networks: VGG \cite{simonyan2014vgg}, GoogleNet \cite{szegedy2015deep}, ResNet \cite{he2015resnet, he2016identity}, FractalNet \cite{larsson2016fractal}, and DenseNet \cite{huang2016dense}, continued to introduce new architectural structures and increase their depth. In each case, the depth of the network seems to contribute to the improved classification accuracy. In particular, it was shown in  \cite{he2015resnet, he2016identity} that deeper networks tended to improve classification accuracy on the common datasets (CIFAR 10, CIFAR 100, and ImageNet).  It is not unusual for DNNs to have thousands of layers!

Although DNNs are widely successful in application, our understanding of their theoretical properties and behavior is limited. In this work, we develop connections between feedforward networks and optimal control problems. These connections are used to construct networks that satisfy some desired stability properties. To test the ideas, we will focus on the image classification problem. Let $\mathcal{D}$ be a set of images which are sampled from $n$ distinct classes. The goal of the classification problem is to learn a function whose output $y\in\R^n$ predicts the correct label associated with the input image $x\in\mathcal{D}$. The $j$-th component of $y$ represents the probability of $x$ being in Class $j$.  It is worth noting that the image classification problem is an example of a high-dimensional problem that can be better solved by DNN than other standard approaches. One possible reason for this is that the mapping from images to labels represented by a neural network may generalize well to new data \cite{bengio2009learning, lecun2015deep}.  
  
As the network depth increases, several issues can occur during the optimization (of network parameters). Take for example the (supervised) image classification problem, where one learns a network by optimizing a cost function over a set of parameters.  Since the parameters are high-dimensional and the problem is non-convex, one is limited in their choice of optimization algorithms \cite{bottou2018optimization}. In addition, the size of the training set can affect the quality and stability of the learned network \cite{bottou2018optimization}. The nonconvexity of the optimization problem may yield many local minimizers, and in \cite{keskar2016large} it was argued that sharp local minimizer could produce networks that are not as generalizable as the networks learned from flatter local minimizers. In \cite{li2017visualizing}, the authors showed that (visually) the energy landscape of ResNet and DenseNet is well-behaved and may be flatter than CNNs without shortcuts. Another potential issue with training parameters of deep networks involves exploding or vanishing gradients, which has been observed in various network architectures  \cite{bengio1994learning}. Some partial solutions have been given by using ReLU as the activation function \cite{ILSVRC15} and by adding identity shortcuts \cite{he2015resnet, he2016identity}. In addition, networks can be sensitive to the inputs in the sense that small changes may lead to misclassification  \cite{szegedy2013intriguing, biggio2013evasion, goodfellow2014generative}. This is one of the motivations for providing a quantitive measure of input-sensitivity in this work.

Recently, there have been several works addressing the architecture of neural networks as the forward flow of a dynamical system. By viewing a neural network as a dynamical system, one may be able to address issues of depth, scale, and stability by leveraging previous work and theory in differential equations.   In \cite{weinan2017proposal}, the connection between continuous dynamical systems and DNNs was discussed.  In \cite{haber2017inverse}, the authors proposed several architectures for deep learning by imposing conditions on the weights in residual layers.  The motivation for the architectures in \cite{haber2017inverse} directly came from the ordinary differential equation (ODE) formulation of ResNets (when there is only one activation per residual layer). For example, they proposed using a Hamiltonian system, which should  make the forward and back propagation stable in the sense that the norms of the features do not change. There could be more efficient ways to compute the back propagation of DNNs based on Hamiltonian dynamics, since the dynamics are time-reversible \cite{chang2017reversible}. Reversible networks have several computationally beneficial properties \cite{gomez2017revnet}; however, layers such as batch normalization \cite{iioffe2015batchnorm} may limit their use. The main idea of batch normalization is to normalize each training mini-batch by reducing its internal covariate shift, which does not preserve the Hamiltonian structure (at least directly).  In a similar direction, ResNet-based architectures can be viewed as a control problem with the transport equation  \cite{li2017deep}.  In \cite{ruthotto2018pde}, the authors designed networks using a symmetric residual layer which is related to parabolic and hyperbolic time-dependent partial differential equations, which produced similar results to the standard ResNet architecture. In \cite{2018arXiv180701083E}, the authors formulated the population risk minimization problem in deep learning as a mean-field optimal control problem, and proved optimality conditions of the Hamilton-Jacobi-Bellman type and the Pontryagin type. It is worth noting that some theoretical arguments connecting a ResNet with one convolution and one activation per residual layer to a first-order ODE are provided in \cite{2018arXiv181011741T}.
 
In image classification, the last operation is typically an application of the softmax function so that the output of the network is a vector that represents the probability of an image being in each class; however, in \cite{wang2018deep} a harmonic extension is used. The idea in \cite{wang2018deep} is to learn an appropriate interpolant as the last layer, which may help to generalize the network to new data.  In \cite{oberman2018lipschitz}, the authors proposed a Lipschitz regularization term to the optimization problem and showed (theoretically) that the output of the regularized network converges to the correct classifier when the data satisfies certain conditions. In addition, there are several recent works that have made connections between optimization in deep learning and numerical methods for partial differential equations, in particular, the entropy-based stochastic gradient descent \cite{chaudhari2018bgd} and a Hamilton-Jacobi relaxation \cite{chaudhari2018deep}. For a review of some other recent mathematical approaches to DNN, see \cite{vidal2017math} and the citations within.

\subsection{Contributions of this Work}

In this work, we connect the post-activation ResNet (Form (a) in \cite{he2016identity}) to an optimal control problem with differential inclusions. We show that the differential system is well-posed and provide explicit stability bounds for the optimal control problem in terms of learnable parameters ({\it i.e.} the weights and biases). In particular,  we provide a growth bound on the norm of the features and a bound on the sensitivity of the features with respect to perturbations in the inputs. These results hold in the continuous-time limit ({\it i.e.} when the depth of the network goes to infinity) and in the discrete setting where one includes all other operations such as batch normalization and pooling.

Since the stability results measure how sensitive the feature space is to perturbations on the input image, these results likely relate to the output accuracy. Based on the theory, we investigate two variants of ResNet that are developed in order to improve the two stability bounds. The variants are constructed by altering the architecture of the post-activation ResNet and the associated optimization problem used in the training phase. We show in the continuous-time limit and in the discrete network that the variants reduce the growth rate bounds by decreasing the constants in the stability conditions. In some cases, the constants become invariant to depth. Computational experiments on the proposed variants show that the accuracy of ResNet is preserved. It is also observed that for the image classification problem, ResNet and its variants monotonically improve accuracy by increasing depth, which is likely related to the well-posedness of the optimal control problem.

\subsection{Overview}

This paper is organized as follows. In Section \ref{sec: NN operations}, we provide definitions and notations for the DNN operations. Section \ref{sec: NN operations} contains mathematical details on the DNN operations in order to make the formal arguments consistent; however, experts in the field can begin at Section \ref{sec: continuous}. In Section \ref{sec: continuous}, we analyze the forward stability of ResNet and its two variants in continuous-time by relating them to optimal control problems with differential inclusions. In Section \ref{sec: discrete}, we prove the forward stability of the variants in the discrete setting, which includes the full network structure. In Section \ref{sec: results}, experimental results are presented and show that the variants preserve the same accuracy as ResNet, with stronger stability bounds theoretically.

\section{Basic Properties of DNN Operations}
\label{sec: NN operations}

Before detailing the connection between ResNet and differential inclusions, we provide some notations and necessary definitions. A neural network consists of a concatenation of layers, which includes an input layer, multiple hidden layers, and an output layer. The input to each layer is typically a multi-dimensional array, which is referred to as the feature. For example, the input layer of a network for image classification can be an RGB image (a feature with three channels). The hidden layers consist of several basic operations: affine transformations, nonlinear activation functions, convolutions, and dimension changing maps.

\subsection{Convolutions}
Since we are concerned with CNNs for imaging, we will first define the notation for ``convolutions" of arrays. 

\begin{definition} \label{def: conv}
Let $x$ be a feature in $\R^{h\times w}$ and $K$ be a filter in $\R^{n\times n}$. The \textit{channel-wise convolution\footnote{Although it is called convolution, Equation \eqref{eq: def conv} actually defines a form of cross-correlation.} of $x$ and $K$}, denoted by $y:=K\otimes x$, is a feature in $\R^{h\times w}$ such that:
\begin{align}
y_{i,j} := 
\begin{cases}
\sum_{\ell,k=-r}^r  \  K_{\ell+r+1,k+r+1}\  x_{i+\ell,j+k}, \quad & \text{if } n \text{ is odd and } r:=(n-1)/2, \\
\sum_{\ell,k=-r+1}^r  \  K_{\ell+r,k+r}\  x_{i+\ell,j+k}, \quad & \text{if } n \text{ is even and } r:=n/2,
\end{cases} \label{eq: def conv}
\end{align}
for all $i=1,2,\dots,h$ and $j=1,2,\dots,w$; that is, each component $y_{i,j}$ of $y$ is obtained via the summation of the component-wise multiplication of $K$ and an $n\times n$ block in $x$ centered at $x_{i,j}$. 

The \textit{channel-wise convolution of $x$ and $K$ with stride size $a$}, denoted by $z:=(K*x)_{|{s=a}}$, is a feature in $\R^{\ceil*{h/a}\times \ceil*{w/a}}$ such that:
\begin{align}
z_{i,j} := 
\begin{cases}
\sum_{\ell,k=-r}^r \  K_{\ell+r+1,k+r+1} \  x_{(ai-a+1)+\ell,(aj-a+1)+k}, \quad & \text{if } n \text{ is odd and } r:=(n-1)/2, \\
\sum_{\ell,k=-r+1}^r\  K_{\ell+r,k+r} \  x_{(ai-a+1)+\ell,(aj-a+1)+k} , \quad & \text{if } n \text{ is even and } r:=n/2,
\end{cases} \label{eq: def conv s2}
\end{align}
for all $i=1,2,\dots,\ceil*{h/a}$ and $j=1,2,\dots,\ceil*{w/a}$.
\end{definition}

When $a=1$, Equation \eqref{eq: def conv s2} coincides with Equation \eqref{eq: def conv}. By using stride size greater than $1$ ({\it i.e.} $s>1$), one can change the spatial dimension of the features between layers. Note that some padding needs to be applied for convolutions, which depends on the (assumed) boundary condition of the feature. Common padding used in channel-wise convolution includes: zero padding, periodic padding, and symmetric padding. When the feature has a depth component (\textit{i.e.} it has multiple channels), 2D convolution is commonly used, which first takes the channel-wise convolution for each channel and then adds up the results depth-wise.

\begin{notation} 
Given a feature $x\in\R^{h\times w\times d}$, let $x_i$ denote the $i$-th channel of $x$, {\it i.e.} 
\begin{align*}
x = \begin{pmatrix}x_1, x_2, \dots, x_d \end{pmatrix}, \quad \text{with } x_i\in\R^{h\times w} \text{ for all } i=1,2,\dots,d,
\end{align*}
and let $x_{i,j,k}$ denote the $(i,j,k)$-th element in $x$.
\end{notation}

\begin{notation}
Given a feature $K\in\R^{n\times n\times d_1\times d_2}$, let $K_{i,j}$ denote the $(i,j)$-th subfilter of $K$, {\it i.e.} 
\begin{align}
K := \begin{pmatrix}
K_{1,1} & K_{1,2} & \cdots & K_{1,d_2} \\ 
K_{2,1} & K_{2,2} & \cdots & K_{2,d_2} \\
\vdots & \vdots & \ddots & \vdots \\
K_{d_1,1} & K_{d_1,2} & \cdots & K_{d_1,d_2} 
\end{pmatrix} \label{eq: def filter}
\end{align}
with $K_{i,j}\in\R^{n\times n}$ for all $i=1,2,\dots,d_1$ and $j=1,2,\dots,d_2$.
\end{notation}

\begin{definition} \label{def: conv2d}
Let $x$ be a feature in $\R^{h\times w\times d_1}$ and $K$ be a filter in $\R^{n\times n\times d_1\times d_2}$. The \textit{2D convolution of $x$ and $K$}, denoted by $y:=K*x$, is a feature in $\R^{h\times w\times d_2}$ such that each channel $y_j\in\R^{h\times w}$ of $y$ is defined as:
\begin{align}
y_j := \sum_{i=1}^{d_1} \  K_{i,j} \otimes x_i, \quad j=1,2,\dots,d_2. \label{eq: def conv2d}
\end{align}
The \textit{2D convolution of $x$ and $K$ with stride size $a$}, denoted by $z:=(K*x)_{|{s=a}}$, is a feature in $\R^{\ceil*{h/a}\times \ceil*{w/a}\times d_2}$ such that each channel $z_j\in\R^{\ceil*{h/a}\times \ceil*{w/a}}$ of $z$ is defined as:
\begin{align}
z_j := \sum_{i=1}^{d_1} \  (K_{i,j} \otimes x_i)_{s=a}, \quad j=1,2,\dots,d_2. \label{eq: def conv2d s2}
\end{align}
\end{definition}

It  can be seen from Definitions \ref{def: conv} and \ref{def: conv2d} that 2D convolution is a linear operation. In the subsequent analysis, we will use the matrix form of 2D convolution. To derive its matrix form, we first define the vectorization operation.

\begin{definition}
Let $x$ be a feature in $\R^{h\times w\times d}$. The \textit{vectorization of $x$}, denoted by $X:=\vect(x)$, is a vector in $\R^{hwd}$ such that
\begin{align}
X_{(k-1)hw+(i-1)w+j} \ = \ x_{i,j,k}, \label{eq: def vec}
\end{align}
for all $i=1,2,\dots,h$, $j=1,2,\dots,w$, and $k=1,2,\dots,d$.
\end{definition}

\begin{remark}
It can be easily verified that the vectorization operation is bijective and the following equalities hold for all $x\in\R^{h\times w\times d}$:
\begin{subequations}
\begin{align}
\|\vect(x)\|_{\ell^2(\R^{hwd})} \ &= \ \|x\|_F, \label{eq: prop vec 2} \\
\|\vect(x)\|_{\ell^\infty(\R^{hwd})} \ &= \ \max_{i\in[h],\ j\in[w],\ k\in[d]}|x_{i,j,k}|, \label{eq: prop vec infty}
\end{align} \label{eq: prop vec}
\end{subequations}
where $ \|\cdot\|_F$ is the Frobenius norm. 
\end{remark}

Let $x$ be a feature in $\R^{h\times w\times d_1}$, $K$ be a filter in $\R^{n\times n\times d_1\times d_2}$, and $y:=K*x$ be a feature in $\R^{h\times w\times d_2}$. Combining Equations \eqref{eq: def conv}, \eqref{eq: def conv2d}, and \eqref{eq: def vec}, one can derive a linear system $Y=AX$ which describes the forward operation $y=K*x$. The general form of $A$ is:
\begin{align}
A = \begin{pmatrix}
A_{1,1} & A_{1,2} & \cdots & A_{1,d_1} \\ 
A_{2,1} & A_{2,2} & \cdots & A_{2,d_1} \\
\vdots & \vdots & \ddots & \vdots \\
A_{d_2,1} & A_{d_2,2} & \cdots & A_{d_2,d_1} 
\end{pmatrix}, \label{eq: A derivation1}
\end{align}
where each $A_{i,j}\in\R^{hw\times hw}$ is a block-wise circulant matrix associated with the channel-wise convolution with $K_{j,i}$ (for all $i=1,2,\dots,d_1$ and $j=1,2,\dots,d_2$). Take, for example, $h=w=4$, $d_1=d_2=2$, and $n=3$, then each block $A_{i,j}$ can be written as:
\begin{align}
A_{i,j} = 
\begin{pmatrix}
U^2_{i,j} & U^3_{i,j} & 0 & U^1_{i,j} \\
U^1_{i,j} & U^2_{i,j} & U^3_{i,j} & 0 \\
0 & U^1_{i,j} & U^2_{i,j} & U^3_{i,j} \\
 U^3_{i,j} & 0 & U^1_{i,j} & U^2_{i,j}
\end{pmatrix}\in\R^{16\times16}, \quad i,j=1,2, \label{eq: A derivation2}
\end{align}
where $U^\ell_{i,j}\in\R^{4\times4}$ for $\ell=1,2,3$. If periodic padding is applied in the forward operation $y=K*x$, then each subblock $U^\ell_{i,j}$ is a circulant matrix defined as:
\begin{align}
U^\ell_{i,j} = 
\begin{pmatrix}
(K_{j,i})_{\ell,2} & (K_{j,i})_{\ell,3} & 0 & (K_{j,i})_{\ell,1} \\
(K_{j,i})_{\ell,1} & (K_{j,i})_{\ell,2} & (K_{j,i})_{\ell,3} & 0 \\
0 & (K_{j,i})_{\ell,1} & (K_{j,i})_{\ell,2} & (K_{j,i})_{\ell,3} \\
(K_{j,i})_{\ell,3} & 0 & (K_{j,i})_{\ell,1} & (K_{j,i})_{\ell,2}
\end{pmatrix}, \quad  i,j=1,2 \text{ and } \ell=1,2,3,\label{eq: A derivation3}
\end{align}
where $(K_{j,i})_{\ell,m}$ denotes the $(\ell,m)$-th element of the $(j,i)$-th subfilter of $K$ (see Equation \eqref{eq: def filter}). Similarly, from Equations \eqref{eq: def conv s2} and \eqref{eq: def conv2d s2}, a linear system $Y=A_{|{s=a}}X$ can be derived to describe the forward operation $y=(K*x)_{{|}s=a}$. One can check that $A_{|{s=a}}\in\R^{\ceil*{h/a}\ceil*{w/a}d_2\times hwd_1}$.

\begin{remark} \label{rem: relate norms}
From Equations \eqref{eq: A derivation1}-\eqref{eq: A derivation3}, we can relate norms between $A$ and $K$. For example, with periodic padding and $a=1$, we have:
\begin{align*}
\|A\|_{\ell^{p,p}}^p  &=  \sum_{i=1}^{hwd_2} \sum_{j=1}^{hwd_1}|A_{i,j}|^p  =  hw\sum_{i,j=1}^n\sum_{\ell=1}^{d_1}\sum_{m=1}^{d_2}|K_{i,j,\ell,m}|^p  =  hw\ \|K\|_{\ell^{p,p}}^p
\end{align*}
for $1\le p<\infty$. Therefore, in the optimization, penalties or constraints on $K$ can be applied to $A$ through some simple (linear) operations on $K$.
\end{remark}

Using the linear system representation, we can define the adjoint of 2D convolution as follows.

\begin{definition}
Let $x$ be a feature in $\R^{h\times w\times d_2}$ and $K$ be a filter in $\R^{n\times n\times d_1\times d_2}$. Assume that zero padding or periodic padding is used. The \textit{adjoint of the 2D convolution of $x$ and $K$}, denoted by $z:=\tilde{K}* x$, is a feature in $\R^{h\times w\times d_1}$ such that $Z=A^TX$, where $X=\vect(x)$, $Z=\vect(z)$, $A$ is the matrix associated with the 2D convolution operation with $K$ defined by Equation \eqref{eq: A derivation1}, and $A^T$ is the (standard) transpose of $A$ in the matrix sense.  The \textit{adjoint filter $\tilde{K}$} is defined to be the filter whose matrix form is $A^T$.  
 \end{definition}

\begin{remark}
Observe that the matrix $A$ in Equation \eqref{eq: A derivation1} is sparse: if $x$ is a feature in $\R^{h\times w\times d_1}$ and $K$ is a filter in $\R^{n\times n\times d_1\times d_2}$, then there are only at most $n^2hwd_1d_2$ nonzero elements in $A$; that is, if $y=K*x$, then each element in $y$ is only locally connected to the elements in $x$. In contrast to convolution, dense multiplication performs a linear operation such that each element in $y$ is connected to all elements in $x$; that is, with the same notations as above, $X\mapsto WX$, where $W$ is a dense matrix in $\R^{m\times hwd_1}$ for some integer $m$ and is referred to as the weight. Layers which include dense multiplications $X\mapsto WX$ are referred to as \textit{fully connected layers}, and they are often used to extract classifiers form images.
\end{remark}

\subsection{Biases and Batch Normalization}

A \textit{bias} $b$ is often added to the result of the above linear operations, for example $x\mapsto K*x+b$. It is commonly assumed that there is one bias term per channel in a convolution layer; that is, if $x\in\R^{h\times w\times d_1}$ and $K\in\R^{n\times n\times d_1\times d_2}$, then $b= \begin{pmatrix}b_1, b_2, \dots, b_d \end{pmatrix}\in\R^{h\times w\times d_2}$, where each channel $b_i\in\R^{h\times w}$ is a constant matrix for $i=1,2,\dots,d_2$.

Batch normalization is an operation which normalizes a batch of features by the computed batch mean and batch variance with an additional (learnable) scale and shift.

\begin{definition}
(from \cite{iioffe2015batchnorm})
Let $\mathcal{B} := \{x^{(1)}, x^{(2)}, \dots, x^{(m)}\}$ be a batch of features. \textit{Batch normalization of $\mathcal{B}$} is defined as:
\begin{align}
B(x^{(i)}; \gamma, \beta) := \dfrac{\gamma(x^{(i)}-\mu)}{\sigma} + \beta, \quad i=1,2,\dots,m, \label{eq: batch norm}
\end{align}
where $\mu := \sum_{i=1}^m x^{(i)}/m$ and $\sigma^2 := \sum_{i=1}^m (x^{(i)}-\mu)^2/m$. 
\end{definition}
The parameters $\gamma$ and $\beta$ in Equation \eqref{eq: batch norm} are often learned in the optimization. When using batch normalization, one does not need to add biases, since it is explicitly computed through $\beta$.

 \subsection{Padding and Pooling} \label{sec: padding pooling}
To capture hierarchal features, it is practical to change the dimension of the features after every few layers. This is accomplished through padding (extension) and pooling (down-sampling) operations.

\begin{definition}
Let $x$ be a feature in $\R^{h\times w\times d_1}$. The \textit{zero padding operator}, {\it i.e.} extension by zero, with parameter $d_2>d_1$, denoted by $E:\R^{h wd_1}\to\R^{hwd_2}$, is defined as:
\begin{align}
E(\vect(x);d_2):=\vect(y), \label{eq: def padding vec}
\end{align}
where $y$ is a feature in $\R^{h\times w\times d_2}$ such that each channel $y_i\in\R^{h\times w}$ of $y$ is defined as:
\begin{align}
y_i :=
\begin{cases}
x_{i-d}, \quad &\text{if } d+1\le i \le d+d_1 \text{ where } d:=\floor*{(d_2-d_1)/2}, \\
0, \quad &\text{otherwise},
\end{cases} \label{eq: def padding}
\end{align}
for $i=1,2,\dots,d_2$.
\end{definition}

\begin{definition}
Let $x$ be a feature in $\R^{h\times w\times d}$. The \textit{2D average pooling operator} with filter size $2\times 2$ and stride size 2, denoted by $P_2: \R^{hwd}\to\R^{\ceil*{h/2}\ceil*{w/2}d}$, is defined as:
\begin{align}
P_2(\vect(x)):=\vect(y), \label{eq: def 2d pooling vec} 
\end{align}
where $y$ is a feature in $\R^{\ceil*{h/2}\times\ceil*{w/2}\times d}$ such that each channel $y_i\in\R^{\ceil*{h/2}\times\ceil*{w/2}}$ is defined as:
\begin{align}
y_i := \dfrac{1}{4} \left(\begin{pmatrix} 1 & 1 \\ 1 & 1 \end{pmatrix} \otimes x_i\right)_{{|}s=2}, \quad i=1,2,\dots,d. \label{eq: def 2d pooling} 
\end{align}
where zero padding is used to perform the convolution.
\end{definition}

Toward the end of the network, the features tend to have a large number of channels, while the number of elements in each channel is small. The last layers of a network often include a global pooling layer, which reduces each channel to its average. 

\begin{definition}
Let $x$ be a feature in $\R^{h\times w\times d}$. The \textit{global average pooling operator}, $P_g: \R^{hwd}\to\R^{d}$, is defined as:
\begin{align*}
P_g(\vect(x)):=y, 
\end{align*}
where $y$ is a vector in $\R^d$ such that each component $y_k$ of $y$ is defined as:
\begin{align*}
y_k&:=\dfrac{1}{hw}\sum_{i=1}^h\sum_{j=1}^w \ x_{i,j,k}, \quad k=1,2,\dots,d. 
\end{align*}
\end{definition}

The following proposition shows that the pooling operators are non-expansive in $\ell^2$ and $\ell^\infty$.

\begin{proposition} \label{prop: pooling}
The pooling operators $P_2$ and $P_g$ are non-expansive in $\ell^2$ and $\ell^\infty$ in the sense that if $x\in\R^{h\times w\times d}$, then
\begin{subequations}
\begin{align}
\|P_2(\vect(x))\|_{\ell^2(\R^{h_1w_1d})}  &\le  \|\vect(x)\|_{\ell^2(\R^{hwd})}, \label{eq: prop 2d pooling 2} \\
\|P_2(\vect(x))\|_{\ell^\infty(\R^{h_1w_1d})}  &\le  \|\vect(x)\|_{\ell^\infty(\R^{hwd})}, \label{eq: prop 2d pooling infty}
\end{align} \label{eq: prop 2d pooling}
\end{subequations}
where $h_1:=\ceil*{h/2}$ and $w_1=\ceil*{w/2}$, and
\begin{subequations}
\begin{align}
\|P_g(\vect(x))\|_{\ell^2(\R^d)}  &\le  \|\vect(x)\|_{\ell^2(\R^{hwd})}, \label{eq: prop global pooling 2} \\
\|P_g(\vect(x))\|_{\ell^\infty(\R^d)}  &\le  \|\vect(x)\|_{\ell^\infty(\R^{hwd})}. \label{eq: prop global pooling infty}
\end{align} \label{eq: prop global pooling}
\end{subequations}
\end{proposition}

The proof of Proposition \ref{prop: pooling} is provided in Appendix \ref{sec: proof}. Similarly, the padding operator is norm preserving with respect to the input and output spaces.
\begin{proposition} \label{prop: padding}
The padding operator $E$ has the following norm preserving  property: if $x\in\R^{h\times w\times d_1}$ and $d_2>d_1$, then
\begin{align}
\|E(\vect(x);d_2)\|_{\ell^p(\R^{hwd_2})}  &=  \|\vect(x)\|_{\ell^p(\R^{hwd_1})}.  \label{eq: prop padding}
\end{align}
for all $p \in[1, \infty]$.
\end{proposition}

\subsection{Activation}
The activation function is a nonlinear function applied to the features. Some common examples of activation functions include: the Sigmoid function $(1+\exp(x))^{-1}$, the hyperbolic tangent function $\tanh(x)$, and the rectified linear unit (ReLU) $ x_+ \equiv \max(x,0)$. In \cite{mishkin2016cnn}, the authors tested various design choices in CNNs; in particular, the compatibility of non-linear activation functions with batch normalization. One observation was that the exponential linear unit (ELU) preforms well without the need of batch normalization, which is defined as: $\alpha(\exp(x)-1)$ if $x<0$ and $x$ if $x\ge0$, where $\alpha>0$. In this paper, we will focus on the ReLU activation function.

\begin{remark}\label{rem: relu}
Using ReLU as the activation function can be viewed as applying a proximal step in the dynamical system that defines the forward propagation. This will be made clear in Section \ref{sec: continuous}. Let ${I}_{\R^{d}_+}$ be the indicator function of the set $\R^{d}_+$, which is defined as:
\begin{align}
{I}_{\R^{d}_+}(x):= 
\begin{cases}
0, \quad &\text{if }  x \in  \R^{d}_+, \\
\infty, \quad &\text{if } x \notin  \R^{d}_+.
\end{cases} \label{eq: indicator}
\end{align}
The proximal operator associated with ${I}_{\R^{d}_+}$ is in fact $\relu$, \textit{i.e.}
\begin{align*}
\prox_{\gamma{I}_{\R^{d}_+}} (x) = \argmin_{y \in \R^{d}}\  \gamma {I}_{\R^{d}_+}(x) + \frac{1}{2} \|x-y\|_{\ell^2(\R^d)}^2= \proj_{\R^{d}_+}(x) = x_+,
\end{align*}
and is independent of $\gamma>0$.
\end{remark}
In addition, ReLU has the following properties.

\begin{proposition} \label{prop: relu}
Let $n\in\N$ and $1\le p\le \infty$. The rectified linear unit is non-expansive and $1$-Lipschitz in $\ell^p(\R^n)$ in the sense that:
\begin{subequations}
\begin{align}
\|x_+\|_{\ell^p(\R^n)} &\le \|x\|_{\ell^p(\R^n)} \label{eq: prop relu1} \\
\|x_+-y_+\|_{\ell^p(\R^n)} &\le \|x-y\|_{\ell^p(\R^n)} \label{eq: prop relu2}
\end{align} \label{eq: prop relu}
\end{subequations}
for all $x,y\in\R^n$.
\end{proposition}

\section{Continuous-time ResNet System} \label{sec: continuous}

The standard (post-activation) form of a residual layer can be written as an iterative update defined by:
\begin{align}
x^{n+1} = ( x^n - \tau \, A^n_2\, \sigma(A^n_1 x^n+b^n_1)+ \tau \,b^n_2)_+, \label{eq: general form discrete}
\end{align}
where $x^n \in \R^d$ is a vector representing the features in layer $n$, $A^n_i \in \R^{d \times d}$ (for $i=1,2$) are the weight matrices, $b^n_i\in \R^d$ (for $i=1,2$) are the biases, and $\sigma$ is some activation function. The parameter $\tau>0$ can be absorbed into the weight matrix $A^n_2$; however, when scaled in this way, the iterative system resembles a forward Euler update applied to some differential equation. The connection between the residual layers (for a single activation function) and differential equations has been observed in \cite{ruthotto2018pde}. 

The second activation used in the (post-activation) form of ResNet leads to a differential inclusion:
\begin{align}
-\dfrac{\text{d}}{\text{dt}}x(t) -  A_2(t)\, \sigma(A_1(t)x(t)+b_1(t)) + b_2(t) &\in \partial{I}_{\R^{d}_+}(x), \label{eq: general form continuous}
\end{align}
where ${I}_{\R^{d}_+}$ is the indicator function of the set $\R^{d}_+$ (see Equation \eqref{eq: indicator}). Including the second activation as ReLU  is similar to imposing the ``obstacle'' $x\geq0$ (element-wise); see for example \cite{lions1979splitting, tran20151, schaefferpenalty} and the citations within. It is possible to show that Equation~\eqref{eq: general form discrete} is a consistent discretization of Equation~\eqref{eq: general form continuous}.  Equation~\eqref{eq: general form discrete} is essentially the forward-backward splitting \cite{singer2009splitting, goldstain2014splitting}, where the projection onto the ``obstacle'' is implicit and the force $A_2(t)\, \sigma(A_1(t)x(t)+b_1(t)) + b_2(t))$ is explicit. 

The time parameter, $t>0$, in Equation~\eqref{eq: general form continuous} refers to the continuous analog of the depth of a neural network (without pooling layers). In the limit, as the depth of a neural network increases, one could argue that the behavior of the network (if scaled properly by $\tau$) should mimic that of a continuous dynamical system. Thus, the training of the network, {\it i.e.} learning $A_i$ and $b_i$ given $x(0)$ and $x(T)$, is an optimal control problem. Therefore, questions on the stability of the forward propagation, in particular, do the features remain bounded and how sensitive are they to small changes in the input image, are also questions about the well-posedness of the continuous control problem.

\subsection{Stability of Continuous-time ResNet}\label{sec:stable continuous}

In this section, we will show that the continuous-time ResNet system is well-posed and that the forward propagation of the features is stable in the continuous-time. First, note that the function ${I}_{\R^{d}_+}$ is convex, and thus its subdifferential $\partial{I}_{\R^{d}_+}(x)$ is monotone and is characterized by a normal cone:
\begin{align*}
\partial{I}_{\R^{d}_+}(x) = \mathcal{N}_{\R^d_+}(x) := \{\xi\in\R^d: \langle \xi,y-x\rangle \le 0 \text{ for all } y\in\R^d_+\}.
\end{align*}
By Remark \ref{rem: relu}, we have $\prox_{\gamma{I}_{\R^{d}_+}} (x) = x_+$. Therefore, Equation \eqref{eq: general form discrete} is indeed a discretization of Equation~\eqref{eq: general form continuous}, where the subdifferential of the indicator function is made implicit by the proximal operator (projection onto $\R^d_+$). We will use both the subdifferential and normal cone interpretation to make the arguments more direct.

Consider differential inclusions of the form:
\begin{align}
-\dfrac{\text{d}}{\text{dt}}x(t)\in\mathcal{N}_{\R^d_+}(x(t)) + F(t,x(t)),\label{eq:inclusion_gen}
\end{align}
which have been studied within the context of optimal control and sweeping processes. The existence of solutions are given by Theorem 1 in \cite{edmond2005sweeping} (see Appendix \ref{sec: aux}). The continuous-time ResNet, characterized by Equation~\eqref{eq: general form continuous}, is a particular case of  Equation~\eqref{eq:inclusion_gen} with the forcing function $F$ set to:
\begin{align*}
F(t,x(t)) := A_2(t)\,\sigma(A_1(t)x(t)+b_1(t)) - b_2(t).
\end{align*} 
Thus, Equation \eqref{eq: general form continuous} is equivalent to:
\begin{align}
-\dfrac{\text{d}}{\text{dt}}x(t)\in\mathcal{N}_{\R^d_+}(x(t)) + A_2(t)\,\sigma(A_1(t)x(t)+b_1(t)) - b_2(t).\label{eq:general form continuous2}
\end{align}
The following result shows that under certain conditions, Equation \eqref{eq:general form continuous2} has a unique absolutely continuous solution in $\R^d_+$.

\begin{theorem} \textbf{(Continuous-time ResNet, Existence of Solutions)} \label{thm: inclusion resnet}\\
Let $c>0$, $x:\R_+ \to \R^{d}$, $A_i: \R_+ \to \R^{d\times d}$, $b_i:\R_+ \to \R^{d}$ (for $i=1,2$), and $\sigma:\R\to\R$ (applied element-wise if the input is $\R^d$). Assume that $\|A_1(t)\|_{\ell^2(\R^d)} \, \|A_2(t)\|_{\ell^2(\R^d)}\le c$ for all $t>0$, and that $\sigma$ is contractive with $\sigma(0)=0$. Then for any $x_0\in\R^d_+$, the following dynamic process:
\begin{align} \begin{cases}
-\dfrac{\text{d}}{\text{dt}}x(t)\in\mathcal{N}_{\R^d_+}(x(t)) + A_2(t)\,\sigma(A_1(t)x+b_1(t)) - b_2(t) \quad \text{a.e. } t>0 \\
x(0) = x_0
\end{cases}
\label{eq: sweeping process}
\end{align}
has one and only one absolutely continuous solution $x\in\R^d_+$.
\end{theorem}

Theorem~\ref{thm: inclusion resnet} shows that in the continuous-time case, there exists only one path in the feature space. Thus, as the number of residual layers increases in a network, we should expect the residual layers to approximate one consistent path from the input to the output. The requirement is that the matrices $A_1$ and $A_2$ are bounded in $\ell^2$, which is often imposed in the training phase via the optimization problem (see Section~\ref{sec: training}).  The stability bounds in the following theorems are derived from the subdifferential interpretation.

\begin{theorem} \textbf{(Continuous-time ResNet, Stability Bounds)} \label{prop: gen stable}\\
With the same assumptions as in Theorem \ref{thm: inclusion resnet}, the unique absolutely continuous  solution $x$ to Equation~\eqref{eq: general form continuous} is stable in the following sense:
\begin{align}
\|x(t)\|_2 &\le \|x(0)\|_2 \, \exp\left(\int_0^t \|A_1(s)\|_2 \,\|A_2(s)\|_2 \, \dd{s}\right) \nonumber \\
&\quad+ \int_0^t \left(\|A_2(s)\|_2\, \|b_1(s)\|_2+\|\left(b_2(s)\right)_+\|_2 \right)\, \exp\left(\int_s^t \|A_1(r)\|_2\, \|A_2(r)\|_2 \, \dd{r}\right) \, \dd{s} \label{eq: general bound 1}
\end{align}
for all $t>0$. In addition, if $y$ is the unique absolutely continuous solution to Equation~\eqref{eq: general form continuous} with input $y(0)$, then for all $t>0$,
\begin{align}
\|x(t)-y(t)\|_2 \leq \|x(0)-y(0)\|_2\,  \exp\left(\int_0^t \|A_1(s)\|_2 \|A_2(s)\|_2 \, \dd{s}\right). \label{eq: general bound 2}
\end{align}
\end{theorem}

Equation~\eqref{eq: general bound 1} provides an upper-bound to the growth rate of the features in the continuous-time network, and Equation~\eqref{eq: general bound 2} shows that the sensitivity of the network to perturbations depends on the size of the weight matrices. Without any additional assumptions on the weights $A_i$ and/or biases $b_i$ (for $i=1,2$) (except for uniform-in-time boundness), the solution to Equation \eqref{eq: general form continuous} and the perturbations can grow exponentially with respect to the depth. By testing a standard ResNet code\footnote{We used the open-sourced code from the TFLearn library on \href{https://github.com/tflearn/tflearn/blob/master/examples/images/residual_network_cifar10.py}{GitHub}.}, we observed that without batch normalization, the norms of the features increase by a factor of 10 after about every 3-4 residual layers. Thus, in very deep networks there could be features with large values, which are typically not well-conditioned. It is interesting to note that with batch normalization, experiments show that the norms of the features grow but not as dramatically. 

In practice, regularization is added to the optimization problem (often by penalizing the norms of the weight matrices) so that the trained network does not overfit the training data. In addition, Theorem~\ref{prop: gen stable} shows that for a deep network, the stability of the continuous-time dynamics depends on the norms of the weight matrices $A_i$. Thus, with sufficient regularization on the weights, the growth rate can be controlled to some extent.

 \subsection{Continuous-time Stability of Variants of ResNet}
There are multiple ways to control the feature-norms in deep ResNets. The results in Section~\ref{sec:stable continuous} indicate that for a general residual layer, the regularization will control the growth rates. Alternatively, by changing the structure of the residual layer through constraints on $A_i$, the dynamics will emit solutions that satisfy smaller growth bound. In Section \ref{sec: results depth}, computational experiments show that the variants produce similar accuracy results to the original ResNet \cite{he2015resnet} with provably tighter bounds.  
  
 We propose two variants on the residual layer, which improve the stability estimates from  Section~\ref{sec:stable continuous}.  We will assume in addition that the activation function $\sigma$ in Equation \eqref{eq: general form discrete} is ReLU. The first form improves the feature-norm bound by imposing that  $A_2(t)\in \R_{+}^{d\times d}$:
\begin{align}
-\dfrac{\text{d}}{\text{dt}}x(t) -  A_2(t)\,\sigma(A_1(t)x(t)+b_1(t)) + b_2(t) &\in \partial{I}_{\R^{d}_+}(x)  \quad \text{with } A_2(t)\in \R_{+}^{d\times d}. \label{eq:inclusion_form1}
\end{align}
The network associated with residual layers characterized by Equation \eqref{eq:inclusion_form1} will be called \textbf{ResNet-D}. This is in reference to the decay of the system when the biases are identically zero for all time. When the biases are non-zero, one can show the following improved bound (as compared to Theorem \ref{prop: gen stable}).

\begin{theorem} \label{prop: form1 stable}
With the same assumptions as in Theorem \ref{thm: inclusion resnet} and that $\sigma$ is ReLU, the unique absolutely continuous  solution $x$ to Equation~\eqref{eq:inclusion_form1} is stable in the following sense:
\begin{align}
\|x(t)\|_2 \le \|x(0)\|_2 + \int_0^t\|\left(b_2(s)\right)_+\|_2\,\dd s \label{eq: improved 1}
\end{align}
for all $t>0$. 
\end{theorem}

Theorem \ref{prop: form1 stable} shows that the continuous-time feature vector does not grow as quickly as the depth of the network increases. In order to improve Equation \eqref{eq: general bound 2}, which measures the sensitivity of the features to changes in the inputs, we impose a symmetric structure to the weights:
\begin{align}
-\dfrac{\text{d}}{\text{dt}}x(t) -  A(t)^T\,\sigma(A(t)x(t)+b_1(t)) + b_2(t) &\in \partial{I}_{\R^{d}_+}(x). \label{eq:inclusion_form2} 
\end{align}
We refer to the network associated with residual layers characterized by Equation \eqref{eq:inclusion_form2} as \textbf{ResNet-S}. The forcing function:
 \begin{align*}
F(t,x)= -A(t)^T\,\sigma(A(t)x(t)+b_1(t))
\end{align*} 
in Equation \eqref{eq:inclusion_form2} was proposed in \cite{chang2017reversible, ruthotto2018pde} and is motivated by parabolic differential equations. Similarly, Equation~\eqref{eq:inclusion_form2} is the nonlinear parabolic differential equation which (under certain conditions) arises from an obstacle problem using the Dirichlet energy \cite{lions1979splitting, tran20151, schaefferpenalty}.  The following result shows that Equation \eqref{eq:inclusion_form2} improves the bounds in Theorem \ref{prop: gen stable}.

\begin{theorem} \label{prop: form2 stable}
With the same assumptions as in Theorem \ref{thm: inclusion resnet} and that $\sigma$ is ReLU, the unique absolutely continuous solution $x$ to Equation~\eqref{eq:inclusion_form2} is stable in the following sense:
\begin{align}
\|x(t)\|_2 &\le \|x(0)\|_2 + \int_0^t \| \left(-A(s)^T\sigma(b_1(s)) + b_2(s) \right)_+\|_2\,\dd s\label{eq: improved 2.1}
\end{align}
for all $t>0$.  In addition, if $y$ is the unique absolutely continuous solution to Equation~\eqref{eq:inclusion_form2} with input $y(0)$, then for all $t>0$,
\begin{align}
\|x(t)-y(t)\|_2 &\leq \|x(0)-y(0)\|_2. \label{eq: improved 2.2}
\end{align}
\end{theorem}
Equation~\eqref{eq: improved 2.2} shows that the features are controlled by perturbations in the inputs, \textit{regardless of the depth}. The proofs of Theorems \ref{thm: inclusion resnet}-\ref{prop: form2 stable} are provided in Appendix \ref{sec: proof}.

\section{Discrete Stability of ResNet-D and ResNet-S}\label{sec: discrete}

Since DNNs are discrete, in this section we provide discrete stability bounds on the features, similar to those in Section~\ref{sec: continuous}. For simplicity, we set all activation functions to ReLU.

\subsection{Architecture of ResNet-D and ResNet-S}

We will discuss the architecture used for the problem of image classification and the associated architecture of ResNet-D and ResNet-S. The base structure of the networks is shown in Figure \ref{fig: network baseline}, which is a variant of the standard architecture for ResNets \cite{he2015resnet}. The input to a network is an image in $\R^{h_1\times w_1 \times d_0}$, and the first layer in the network is a convolution layer (shown in Figure \ref{fig: network convolution}), which increases the depth of the input image to $d_1$. The convolution layer is followed by a stack of $m$ residual layers, which take one of the two forms detailed in Figures \ref{fig: resnet form1} and \ref{fig: resnet form2}. The residual block is followed by a 2D pooling layer (shown in Figure \ref{fig: network 2d pooling}), which halves the resolution and doubles the depth of the incoming feature ({\it i.e.} $h_2=\ceil*{h_1/2}$, $w_2=\ceil*{w_1/2}$, and $d_2=2d_1$). The resulting feature is then processed by a stack of $m-1$ residual layers, a 2D pooling layer ({\it i.e.} $h_3=\ceil*{h_2/2}$, $w_3=\ceil*{w_2/2}$, and $d_3=2d_2$), and another stack of $m-1$ residual layers. Finally, we reduce the dimension of the resulting feature by adding a global average pooling layer (shown in Figure \ref{fig: network global pooling}) and a fully connected layer (shown in Figure \ref{fig: network fully connected}).

\begin{figure}[t!]
\centering
\includegraphics[scale=0.8]{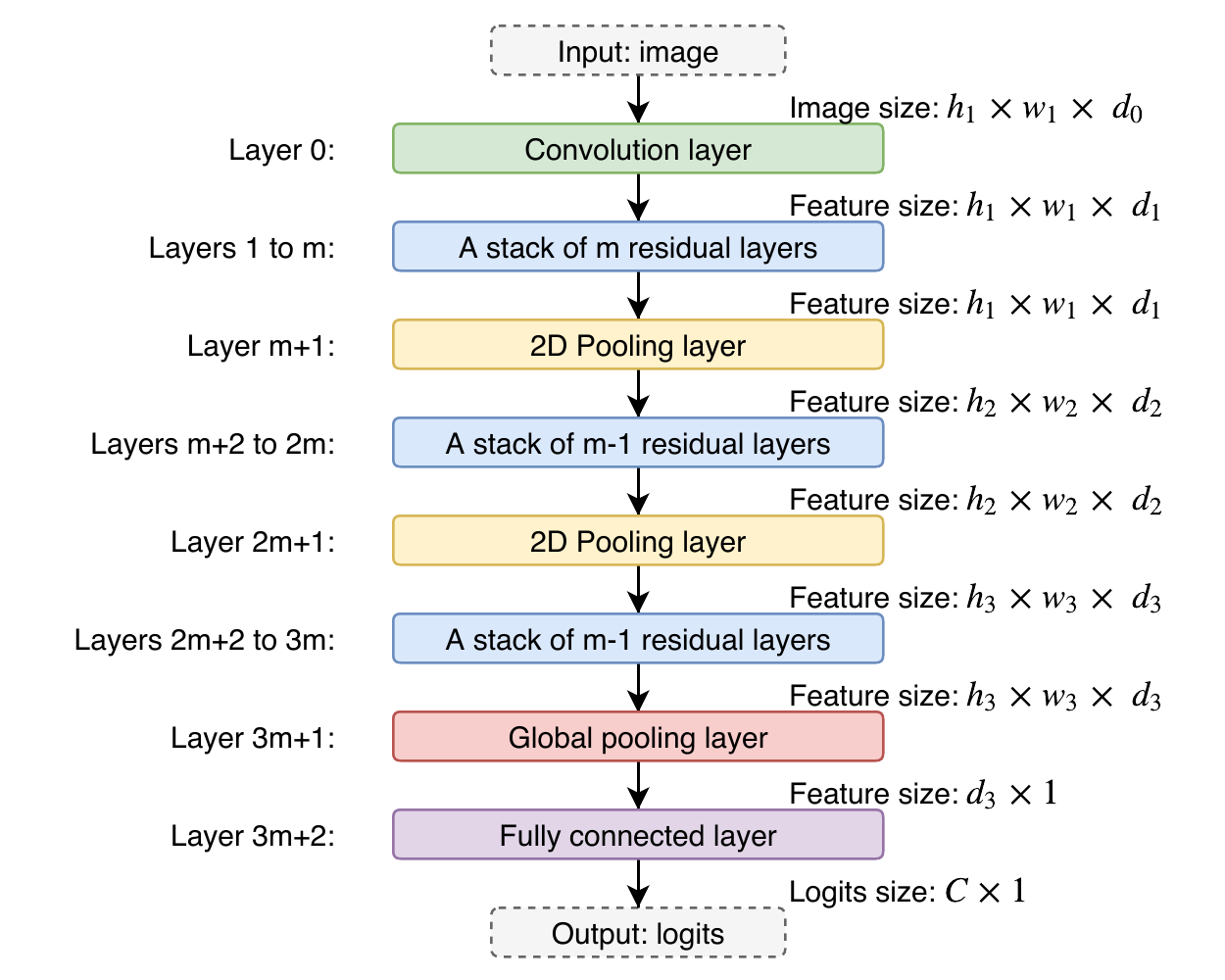}
\caption{Architecture of ResNet-D and ResNet-S for the image classification problem. The input image is of size $h_1\times w_1 \times d_0$, and is contained in a dataset with $C$ classes. The dimension of the features is changed through the network, where $h_{i+1}=\ceil*{h_i/2}$, $w_{i+1}=\ceil*{w_i/2}$, and $d_{i+1}=2d_i$ (for $i=1,2$).}
\label{fig: network baseline}
\end{figure}

\begin{figure}[b!]
\centering
\begin{minipage}[c]{0.4\linewidth}
\centering
\subfloat[The ResNet-D layer (Eq. \eqref{eq: layer resnet1}). \label{fig: resnet form1}]{\includegraphics[scale=0.72]{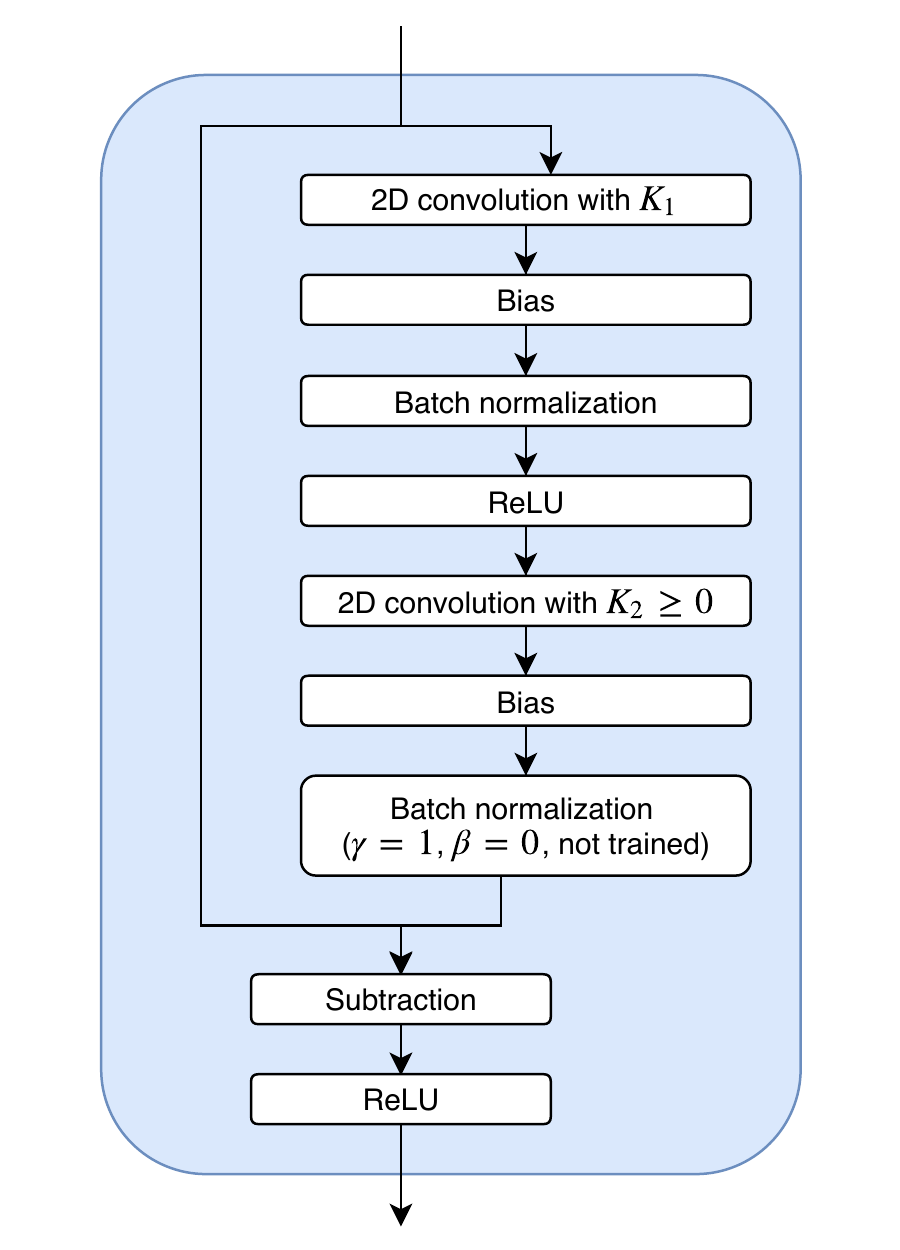}}\par
\subfloat[The ResNet-S layer (Eq. \eqref{eq: layer resnet2}). \label{fig: resnet form2}]{\includegraphics[scale=0.72]{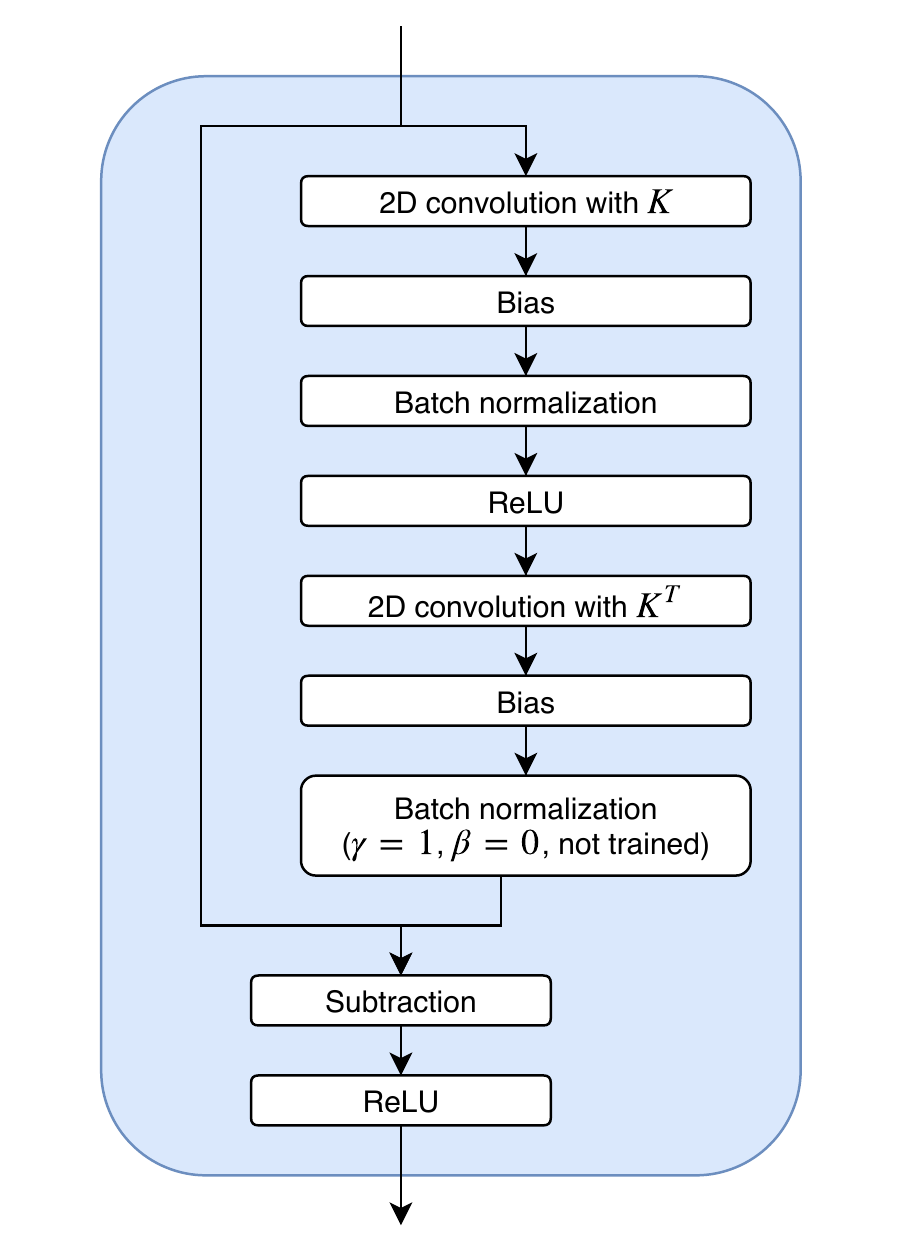}}
\end{minipage}%
\hfill
\begin{minipage}[c]{0.5\linewidth}
\centering
\subfloat[The convolution layer (Eq. \eqref{eq: layer conv}). \label{fig: network convolution}]{\includegraphics[scale=0.72]{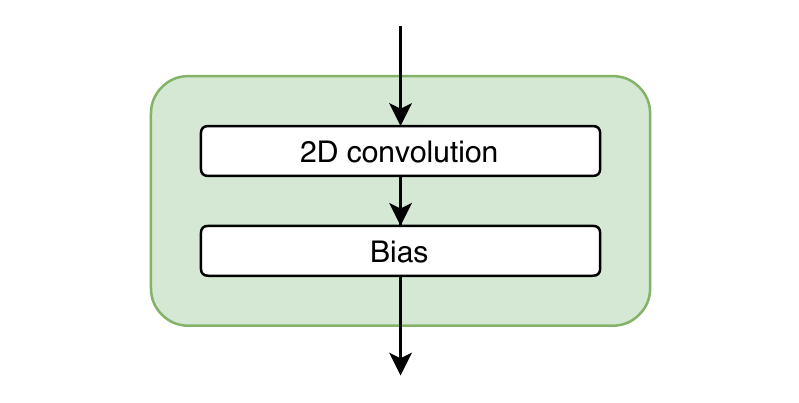}}\par 
\subfloat[The 2D pooling layer (Eq. \eqref{eq: layer 2d pool}). \label{fig: network 2d pooling}]{\includegraphics[scale=0.72]{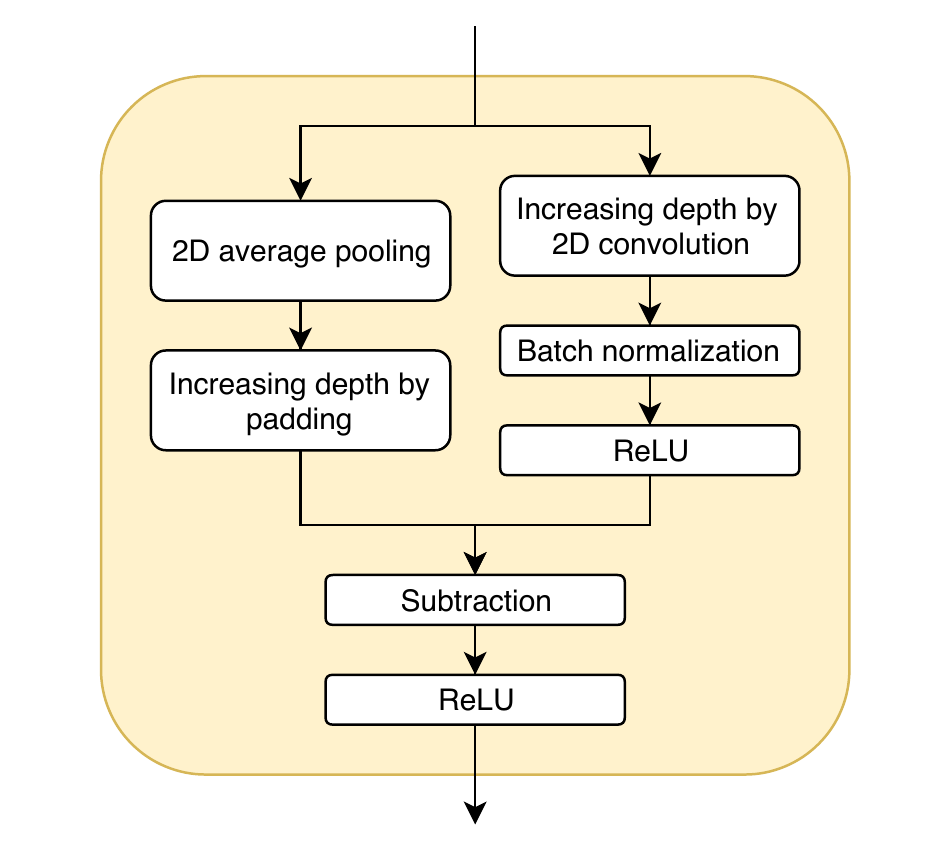}}\par 
\subfloat[The global pooling layer  (Eq. \eqref{eq: layer global pool}). \label{fig: network global pooling}]{\includegraphics[scale=0.72]{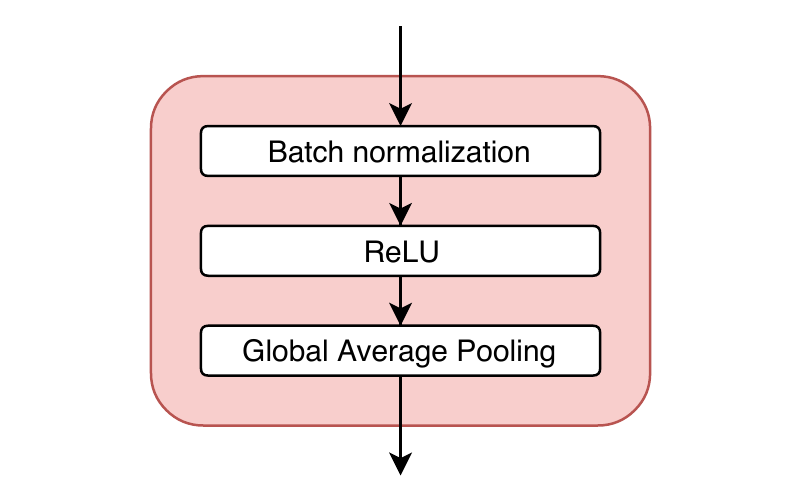}}\par
\subfloat[The fully connected layer (Eq. \eqref{eq: layer dense}). \label{fig: network fully connected}]{\includegraphics[scale=0.72]{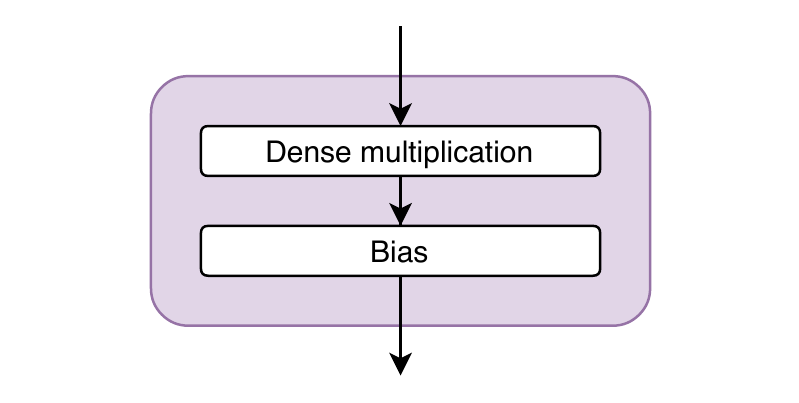}}
\end{minipage}

\caption{Details of the layers in ResNet-D and ResNet-S.}
\label{fig: network layers}
\end{figure}

Let $x^0$ be the vector representing the input to the network, {\it i.e.} $x^0\in\R^{h_1 w_1 d_0}$. The equations that charaterize the layers in Figure \ref{fig: network layers} are defined as follows:
\begin{align}
&\text{the ResNet-D layer: } \quad && x^{n+1} := ( x^n - A^n_2(A^n_1x^n+b^n_1)_+ + b^n_2)_+ \quad \text{ with } A_2^n\ge0, \label{eq: layer resnet1} \\ 
&\text{the ResNet-S layer: } \quad && x^{n+1} := ( x^n - (A^n)^T(A^nx^n+b^n_1)_+ + b^n_2)_+, \label{eq: layer resnet2} \\ 
&\text{the convolution layer:} \quad && x^{n+1} := A^nx^n + b^n, \label{eq: layer conv} \\ 
&\text{the 2D pooling layer: } \quad && x^{n+1} := \left(E(P_2(x^n)) - \left((A^n)_{|{s=2}}\,x^n+b^n\right)_+\right)_+, \label{eq: layer 2d pool} \\ 
&\text{the global pooling layer: }  \quad && x^{n+1} := P_g\left((x^n)_+\right), \label{eq: layer global pool} \\ 
&\text{the fully connected layer:} \quad && x^{n+1} := W^nx^n + b^n, \label{eq: layer dense} 
\end{align}
where $x^n$ is the input to Layer $n$,  $A^n$ is the matrix associated with the 2D convolution operation with $K^n$ in Layer $n$ (when applicable), $b^n$ is the bias in Layer $n$, and $W^n$ is the weight matrix in the fully connected layer. 

The forward propagation of the network is shown in Figures \ref{fig: propagation inf} and \ref{fig: propagation 2}, which display (channel-wise) the output feature of the indicated layer/block of Network-D and Network-S, respectively. As an example, the input image is a hand-written digit ``2" from the MNIST dataset. The first convolution layer (Layer 1) returns low-level features of the digit (Figures \ref{fig: forward 1-1} and \ref{fig: forward 2-1}). The low-level features are then processed by a stack of residual layers (Layers 1 to $m$), which yields mid-level features of the digit (Figures \ref{fig: forward 1-2} and \ref{fig: forward 2-2}). The mid-level features are then downsampled by a 2D pooling layer (Layer $m+1$) and processed by a stack of residual layers (Layers $m+2$ to $2m$), which yields high-level features of the digit (Figures \ref{fig: forward 1-3} and \ref{fig: forward 2-3}). Similarly, after a 2D pooling layer (Layer $2m+1$) and a stack of ResNet layers (Layers $2m+2$ to $3m$), the high-level features become linearly separable classifiers (Figures \ref{fig: forward 1-4} and \ref{fig: forward 2-4}). The global pooling layer (Layer $3m+1$) and the fully connected layer (Layer $3m+2$) convert the linearly separable classifiers to a vector that can be used to extract a predicted probability distribution of the input. For example, the predicted probability distributions in Figures \ref{fig: forward 1-6} and \ref{fig: forward 2-6} are obtained by applying the softmax normalization function (see Definition \ref{def: softmax}) to the output of the fully connected layer, where the value of the $i$-th bar represents the predicted probability that the input digit is $i$ (for $i=1,2,\dots,10$).

Note that the mid-level features resemble images filtered by edge detectors, similar to CNNs and the standard ResNet. Experimentally, we see that a ResNet-D layer produces a kernel $K_1$ which looks like a gradient stencil and a kernel $K_2$ which acts as a rescaled averaging filter. Thus, the first block in ResNet-D resembles a \textit{non-linear (possibly non-local) transport system}. The non-locality comes from the smoothing process determined by $K_2$. In ResNet-S, since the kernels $K$ from the first residual block are gradient-like stencils, the first block in ResNet-S resembles a \textit{nonlinear diffusive system}. 

\begin{figure}[b!]
\centering
\begin{minipage}[c]{0.48\linewidth}
\centering
\subfloat[Input.]{\makebox[2in][c]{\includegraphics[scale=0.95]{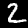}}}\par
\subfloat[Output of the convolution layer.\label{fig: forward 1-1}]{\includegraphics[scale=0.95]{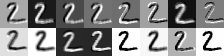}}\par
\subfloat[Output of the first ResNet block.\label{fig: forward 1-2}]{\includegraphics[scale=0.95]{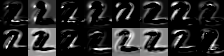}}\par
\subfloat[Output of the second ResNet block.\label{fig: forward 1-3}]{\includegraphics[scale=0.95]{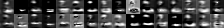}}\par
\subfloat[Output of the third ResNet block.\label{fig: forward 1-4}]{\includegraphics[scale=0.95]{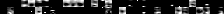}}\par
\subfloat[Predicted probability distrubution.\label{fig: forward 1-6}]{\includegraphics[scale=0.5]{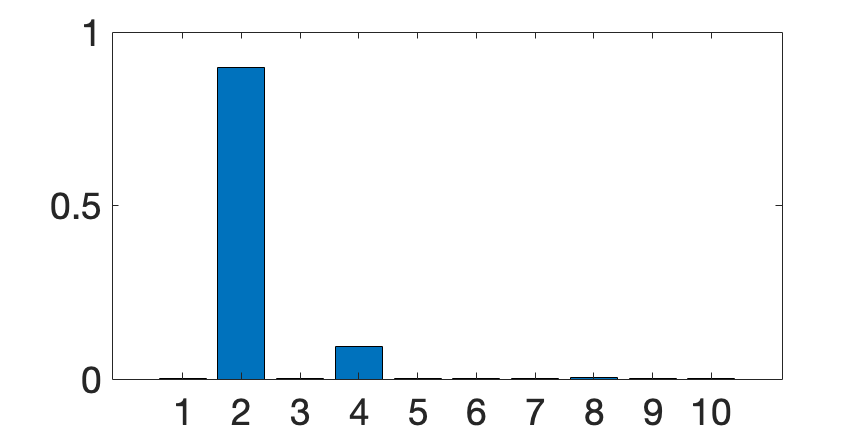}}\par
\caption{Forward propagation of ResNet-D. }
\label{fig: propagation inf}
\end{minipage}
\hfill
\begin{minipage}[c]{0.48\linewidth}
\centering
\subfloat[Input.]{\makebox[2in][c]{\includegraphics[scale=0.95]{m1n3_layer0.png}}}\par
\subfloat[Output of the convolution layer.\label{fig: forward 2-1}]{\includegraphics[scale=0.95]{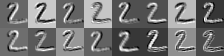}}\par
\subfloat[Output of the first ResNet block.\label{fig: forward 2-2}]{\includegraphics[scale=0.95]{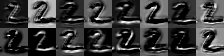}}\par
\subfloat[Output of the second ResNet block.\label{fig: forward 2-3}]{\includegraphics[scale=0.95]{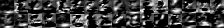}}\par
\subfloat[Output of the third ResNet block.\label{fig: forward 2-4}]{\includegraphics[scale=0.95]{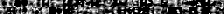}}\par
\subfloat[Predicted probability distribution.\label{fig: forward 2-6}]{\includegraphics[scale=0.5]{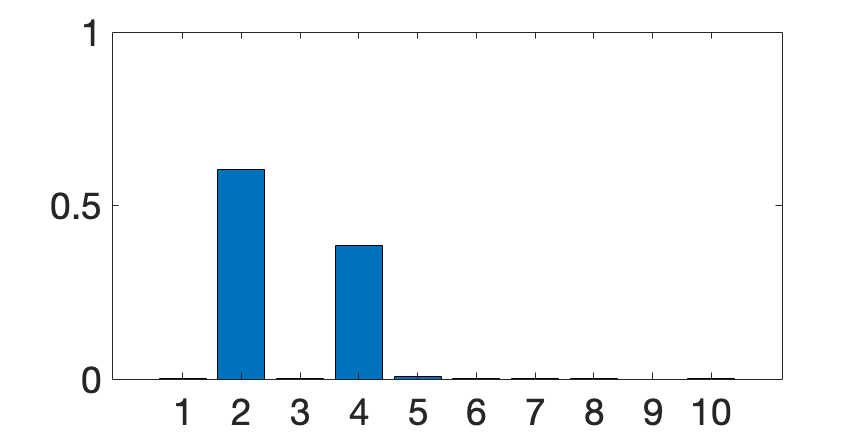}}\par
\caption{Forward propagation of ResNet-S.}
\label{fig: propagation 2}
\end{minipage}%
\end{figure}

\subsection{Forward Stability of ResNet-D and ResNet-S} \label{sec: discrete analysis}

The stability of forward propagation through ResNet-D and ResNet-S can determine both the sensitivity of the network to changes in the inputs and the level of consistency in various computations. If the norms of the weight matrices are small enough, then both the output of the network and changes in the features can be controlled by the inputs. In particular, we have the following (discrete) stability results for ResNet-D and ResNet-S. 

\begin{theorem}\textbf{(Forward Stability, ResNet-D)}\label{thm: res1 net} \\
Consider a network defined in Figure \ref{fig: network baseline}, where the ResNet layers are defined in Figure \ref{fig: resnet form1}. Let $x^0$ be the vectorization of the input to the network, {\it i.e.} $x^0\in\R^{h_1w_1d_0}$, and for each filter $K^n$ in Layer $n$ (when applicable), let $A^n$ be the matrix associated with the 2D convolution operation with $K^n$. Assume that
\begin{subequations}
\begin{align}
&\|A^0\|_{\ell^\infty(\R^{h_1w_1d_0})\to\ell^\infty(\R^{h_1w_1d_1})}\le1, \\
&\|W^{3m+2}\|_{\ell^\infty(\R^{d_3})\to\ell^\infty(\R^{C})}\le1.
\end{align} \label{thm: res1 net constraint}
\end{subequations}
Let $x^n$ be the input to the $n$-th layer and $x^N$ be the output of the network, where $N:=3m+3$.  Then the network is $\ell^\infty$-stable in the sense that:
\begin{align}
\|x^N\|_{\ell^\infty(\R^C)} \le \|x^0\|_{\ell^\infty(\R^{h_1w_1d_0})} + c(b^0,b^1,\dots,b^{N-1}), \label{eq: res1 net}
\end{align}
where $c(b^0,b^1,\dots,b^{N-1})$ is a constant depending on the $\ell^\infty$ norms of the biases in the network; see Equation \eqref{eq: res1 net constant}. If $y^0\in\R^{h_1w_1d_0}$ is the vectorization of another input, then:
\begin{align}
\|x^N-y^N\|_{\ell^2(\R^C)} \le a(A^0,A^1,\dots,W^{N-1}) \,\|x^0-y^0\|_{\ell^2(\R^{h_1w_1d_0})}, \label{eq: res1 net2}
\end{align}
where $a(A^0,A^1,\dots,W^{N-1})$ is a constant depending on the $\ell^2$ norms of the filters and weights in the network; see Equation \eqref{eq: res1 net2 constant}.
\end{theorem}

Note that the bounds on the growth of the features, Equation \eqref{eq: res1 net}, do not directly depend on the filters and weights, since the system (without the biases) decays.  The sensitive bound, Equation \eqref{eq: res1 net2}, depends on the $\ell^2$ norms of the filters and weights, which are controlled by the regularizer (see Table~\ref{tab: regularizer}). For ResNet-S, we can improve the constant in the sensitivity bound as follows.

\begin{theorem}\textbf{(Forward Stability, ResNet-S)} \label{thm: res2 net} \\
Consider a network defined in Figure \ref{fig: network baseline}, where the residual layers are defined in Figure \ref{fig: resnet form2}. Let $x^0$ be the vectorization of the input to the network, {\it i.e.} $x^0\in\R^{h_1w_1d_0}$, and for each filter $K^n$ in Layer $n$ (when applicable), let $A^n$ be the associated matrix. Assume that
\begin{subequations}
\begin{align}
&\|A^0\|_{\ell^2(\R^{h_1w_1d_0})\to\ell^2(\R^{h_1w_1d_1})}\le1, \\
&\|W^{3m+2}\|_{\ell^2(\R^{d_3})\to\ell^2(\R^{C})}\le1,
\end{align} \label{thm: res2 net constraint}
\end{subequations}
and that for each filter $K^n$ in a residual layer:
\begin{align}
\|A^n\|_{\ell^2}\le\sqrt{2}, \label{thm: res2 net constraint2}
\end{align}
where $\|\cdot\|_{\ell^2}$ denotes the induced (matrix) norm. Let $x^n$ be the input of the $n$-th layer and $x^N$ be the output of the network, where $N:=3m+3$. Then the network is $\ell^2$-stable in the sense that:
\begin{align}
\|x^N\|_{\ell^2(\R^C)} \le \|x^0\|_{\ell^2(\R^{h_1w_1d_0})} +  c(b^0,b^1,\dots,b^{N-1}), \label{eq: res2 net}
\end{align}
where $c(b^0,b^1,\dots,b^{N-1})$ is a constant depending on the $\ell^2$ norms of the biases in the network; see Equation \eqref{eq: res2 net constant}. If $y^0\in\R^{h_1w_1d_0}$ is the vectorization of another input, then:
\begin{align}
\|x^N-y^N\|_{\ell^2(\R^C)} \le a(A^0,A^1,\dots,W^{N-1}) \|x^0-y^0\|_{\ell^2(\R^{h_1w_1d_0})}, \label{eq: res2 net2}
\end{align}
where $a(A^0,A^1,\dots,W^{N-1})$ is a constant  independent of the depth of the residual block; see Equation \eqref{eq: res2 net2 constant}.
\end{theorem}

Equation \eqref{eq: res2 net2} is useful since it implies that, as long as one constrains the norms of the filters such that $\|A^n\|_{\ell^2}\le\sqrt{2}$, the network will be stable for arbitrarily many residual layers. The proofs of Theorems \ref{thm: res1 net} and \ref{thm: res2 net} are provided in Appendix \ref{sec: proof}.

\begin{remark}
By Equation \eqref{eq: batch norm}, given an input feature $x$, the output $y$ of the following concatenation of operations: 
\begin{align*}
\text{2D convolution} \rightarrow \text{bias} \rightarrow \text{batch normalization}
\end{align*}
is obtained by:
\begin{align}
y := \dfrac{\gamma(Ax+b-\mu)}{\sigma} + \beta, \label{eq: rem batch norm}
\end{align}
where the constants $\mu$ and $\sigma$ depend on the mini-batch containing $x$. For the forward propogation, if we set $\tilde A:=\gamma A/\sigma$ and $\tilde b:= \gamma(b-\mu)/\sigma+\beta$, then Equation \eqref{eq: rem batch norm} can be rewritten as:
\begin{align*}
y := \tilde A x + \tilde b.
\end{align*}
To make sure that $A$ satisfies the constraints of $\tilde A$, for example the correct sign in Equation \eqref{eq: layer resnet2}, we fix the parameters $\gamma$ and $\beta$ in the second batch normalization in the residual layers ($\gamma=1$ and $\beta=0$ for the second batch normalization in Figure \ref{fig: resnet form2}).
\end{remark}

\subsection{Training Parameters in ResNet-D and ResNet-S}
\label{sec: training}

A commonly used loss function for the image classification problem is the cross entropy loss function, which is a metric to measure certain distance between two probability distributions.

\begin{definition}
Let $n\in\N$ and $u,v\in\R^{n}$ be two probability distributions, \textit{i.e.} $0\le u_i,v_i\le 1$ for all $i=1,2,\dots,n$ and $\sum_{i=1}^n u_i=\sum_{i=1}^n v_i=1$. The \textit{cross entropy between $u$ and $v$} is defined as:
\begin{align*}
H(u,v) := -u^T\log(v). 
\end{align*}
\end{definition}

Let $\mathcal{D}$ be a dataset with $C$ classes of images.  Given an input image $x^0\in\mathcal{D}$ to the network defined in Figure \ref{fig: network baseline}, let $y\in\R^C$ be the one-hot encoding label vector associated with $x^0$ ({\it i.e.} $y_i=1$ if $x^0$ is of Class $i$, and $y_i=0$ otherwise), and let $x^N\in\R^C$ be the output of the network. The label vector $y$ can be considered as the true distribution of $x^0$ over the $C$ possible classes. To obtain a predicted distribution of $x^0$ from the network and compare it with $y$, we apply the softmax normalization function to the output $x^N$ of the network, so that the loss to be minimized for each input $x^0\in\mathcal{D}$ is $H(y, S(x^N))$.

\begin{definition} \label{def: softmax}
Let $n\in\N$ and $u\in\R^{n}$. The \textit{softmax normalization of $u$} is a vector in $\R^n$ such that:
\begin{align*}
S(u)_i := \dfrac{\exp(u_i)}{\sum_{j=1}^n\exp(u_j)}, \quad i=1,2,\dots, n.
\end{align*}
\end{definition}

Let $\mathcal{I}$ be the index set for the layers in the network. Given an index $n\in\mathcal{I}$, let $K^n$, $b^n$, $\gamma^n$, $\beta^n$, and $W^n$ be the filter, bias, scale, shift, and weight in the $n$-th layer, respectively (when applicable). To minimize the classification error of the network, we solve the following optimization problem:
\begin{align} 
\min_{\substack{K^n, b^n, \gamma^n, \beta^n, W^n \\ \text{for } n\in\mathcal{I}}}&  \ \ \sum_{x^0\in\mathcal{D}} H(y, S(x^N)) + \sum_{n\in\mathcal{I}}R_n,
\label{eq: opt form1}
\end{align}
where $R_n$ is the regularizer (penalty) for the $n$-th layer, which is added to prevent overfitting the training data and to impose the constraints discussed in Section \ref{sec: discrete analysis}. The regularizer for each layer is listed in Table \ref{tab: regularizer}. To impose the constraints in Theorem \ref{thm: res1 net}, we regularize the $\ell^{1,1}$ norms of the filter and weights. To impose the constraints in Theorem \ref{thm: res2 net}, the Frobenius norms of the filters and weights are regularized.  The element-wise constraints $A^n_2\ge0$ in Equation \eqref{eq: layer resnet1} is imposed directly by adding the indicator function $I_{K^n_2\ge0}$ to the regularizer. 

{\tabulinesep=1.5mm
\begin{table}[b!]
  \caption{Regularizers in the optimization problem defined in Equation \eqref{eq: opt form1}. Definitions of the layers are provided in Equations \eqref{eq: layer resnet1}-\eqref{eq: layer dense}. The indicator function $I_{K^n_2\ge0}$ represents the constraint $K^n_2\ge0$, and $\alpha_n$ are some nonnegative constants.}
  \label{tab: regularizer}
  \centering
  \begin{tabu}{|c|c|c|}  \hline
  Layer(s) 				& ResNet-D 					& ResNet-S	\\ \hline
the convolution layer		& $R_n = \alpha_n\|\vect(K^n)\|_{\ell^{1}}$ 	& $R_n = \dfrac{\alpha_n}{2}\|\vect(K^n)\|^2_{\ell^2}$ \\ \hline
the ResNet layers		& $R_n = \dfrac{\alpha_n}{2}\left(\|\vect(K^n_1)\|^2_{\ell^2}+\|\vect(K^n_2)\|^2_{\ell^2}\right) + I_{K^n_2\ge0}$ & $R_n = \dfrac{\alpha_n}{2}\|\vect(K^n)\|^2_{\ell^2}$ \\ \hline
the 2D pooling layers	& $R_n = \dfrac{\alpha_n}{2}\|\vect(K^n)\|^2_{\ell^2}$	& $R_n = \dfrac{\alpha_n}{2}\|\vect(K^n)\|^2_{\ell^2}$ \\ \hline
the fully connected layer	& $R_n = \alpha_n\|\vect(W^n)\|_{\ell^{1}}$	& $R_n = \dfrac{\alpha_n}{2}\|\vect(W^n)\|^2_{\ell^2}$ \\ \hline
  \end{tabu}
  \end{table}}

\section{Computational Experiments} \label{sec: results}

We test the two proposed networks on the CIFAR-10 and CIFAR-100 datasets. The CIFAR-10 (CIFAR-100, respectively) dataset consists of 60,000 RGB images of size $32\times32$ in 10 classes (100 classes, respectively), with 50,000 training images and 10,000 test images. The data are preprocessed and augmented as in \cite{he2015resnet}.

All filters in the network are of size $3\times3$, and we assume that the input feature to each layer satisfies the periodic boundary condition. For the first convolution layer, the filters are initialized using the uniform scaling algorithm \cite{sussillo2014uniform}; for the residual layers and the 2D pooling layers, the filters are initialized using the variance scaling algorithm \cite{he2015weight}.  The weight in the fully connected layer is initialized with values drawn from a normal distribution $\mathcal{N}(0, \sigma^2)$, where $\sigma=(d_3C)^{-1}$, except that values with magnitude more than $2\sigma$ are discarded and redrawn ({\it i.e.} the truncated normal distribution). Note that in \cite{du2018proof} and the citations within, it was shown that under certain conditions on a neural network, randomly initialized gradient descent applied to the associated optimization problem converges to a globally optimal solution at a linear convergence rate.

The biases in the network are initialized to be zero. The batch normalization parameters $\gamma$ and $\beta$, if trained, are initialized to be 1 and 0, respectively.  The regularization parameter $\alpha_n$ in Table \ref{tab: regularizer} is set to be $10^{-4}$ for all $n$.

The network is trained using the mini-batch gradient descent algorithm, with mini-batch size equal to 128 ({\it i.e.} 391 steps/epoch). The initial learning rate is 0.1, and is divided by 10 after every $T_0$ training steps. The training process is terminated after $T_1$ training steps.  Our focus in the experiments is to examine the stability of the variants of ResNet. We would like to demonstrate that the variants achieve similar accuracy. Thus, we fix the hyperparameters (listed in Table \ref{tab: para optimizer}) and do not tune them to the data. Better results can be achieved with tuning. The network is validated on the test images after every 500 training steps. 

\begin{table}[b!]
  \caption{Hyperparameters associated with the optimizer. The learning rate is divided by 10 after every $T_0$ training steps. The training process is terminated after $T_1$ training steps.}
  \label{tab: para optimizer}
  \centering
  \begin{tabular}{|c|c|c|c|}  \hline
   				& Dataset 		& Learning rate decay steps $T_0$	& Number of training steps $T_1$	 	\\ \hline
   \multirow{ 2}{*}{ResNet-D} 	& CIFAR-10 		& 24K (about 62 epochs)	& 70K (about 180 epochs)		 \\ \cline{2-4}
  					& CIFAR-100 		& 24K (about 62 epochs)		& 70K (about 180 epochs)		  \\ \hline 
  \multirow{ 2}{*}{ResNet-S} 	& CIFAR-10 	& 32K (about 82 epochs)		& 93.5K (about 240 epochs)	 \\ \cline{2-4}
  					& CIFAR-100	& 24K (about 62 epochs)		& 70K (about 180 epochs)		 \\ \hline 
  \end{tabular}
  \end{table}

In Table \ref{tab: parameter}, we list the depth of the network and the number of trainable parameters in the optimization problem with a few different values of $m$ (where $m$ is the size of the first residual block). Here, the depth of a network is considered to be the number of (unique) filters and weights in the network; for example, each ResNet-D layer (Figure \ref{fig: resnet form1}) contains two filters, and each ResNet-S layer (Figure \ref{fig: resnet form2}) contains only one filter.

\begin{table}[b!]
  \caption{Depth of the network and number of trainable parameters in the optimization problem. The depth counts the number of filters and weights in the network. The trainable parameters in Equation \eqref{eq: opt form1} include all elements in the filters, weights, and biases, and the parameters in all batch normalization (if trained).}
  \label{tab: parameter}
  \centering

\begin{minipage}[c]{0.5\linewidth}
  \centering
   \subfloat[ResNet-D.\label{tab: form1 depth}]{
  \begin{tabular}{|c|c|c|c|}  \hline
\multirow{2}{*}{$m$} & \multirow{2}{*}{Depth} 	& \multicolumn{2}{c|}{Trainable parameters}  \\ \cline{3-4}
& & CIFAR-10 & CIFAR 100 \\ \hline
3 & 18		& 0.223M	& 0.229M \\ \hline
6 & 36		& 0.514M	& 0.520M \\ \hline
9 & 54		& 0.805M	& 0.810M \\ \hline
12 & 72		& 1.100M	& 1.101M \\ \hline
  \end{tabular}}
   \end{minipage}%
\hfill
\begin{minipage}[c]{0.5\linewidth}
  \centering
  \subfloat[ResNet-S.\label{tab: form2 depth}]{
  \begin{tabular}{|c|c|c|c|}  \hline
\multirow{2}{*}{$m$} & \multirow{2}{*}{Depth} 	& \multicolumn{2}{c|}{Trainable parameters}  \\ \cline{3-4}
& & CIFAR-10 & CIFAR 100 \\ \hline
3 & 11		& 0.124M	& 0.130M \\ \hline
6 & 20		& 0.270M	& 0.275M \\ \hline
9 & 29		& 0.416M & 0.421M \\ \hline
12 & 38		& 0.561M	& 0.567M \\ \hline
  \end{tabular}}
   \end{minipage}%
  \end{table}

\subsection{Effect of Depth on Test Accuracy} \label{sec: results depth}

We train the network with different depths and analyze the effect of the depth on test accuracy.  The resulting test accuracies over the training steps are shown in Figure \ref{fig: result depth}. In particular, we calculate the average of the test accuracy in the last 5,000 training steps and list the results in Table \ref{tab: result depth}.  It can be seen from Figure \ref{fig: result depth} and Table \ref{fig: result depth} that the test accuracy of both ResNet-D and ResNet-S increases as the network goes deeper (without any hyperparameter tuning). This result is consistent with the observation in \cite{he2015resnet} that a deeper ResNet tends to have higher test accuracy. The monotone improvement of accuracy in depth is likely related to the well-posedness of the optimal control problem (Equation \eqref{eq: sweeping process}).

\begin{figure}[b!]
  \centering
  \begin{minipage}[c]{0.5\linewidth}
\centering
  \subfloat[ResNet-D on CIFAR-10. \label{fig: train1 cifar10}]{
  \includegraphics[scale=0.35]{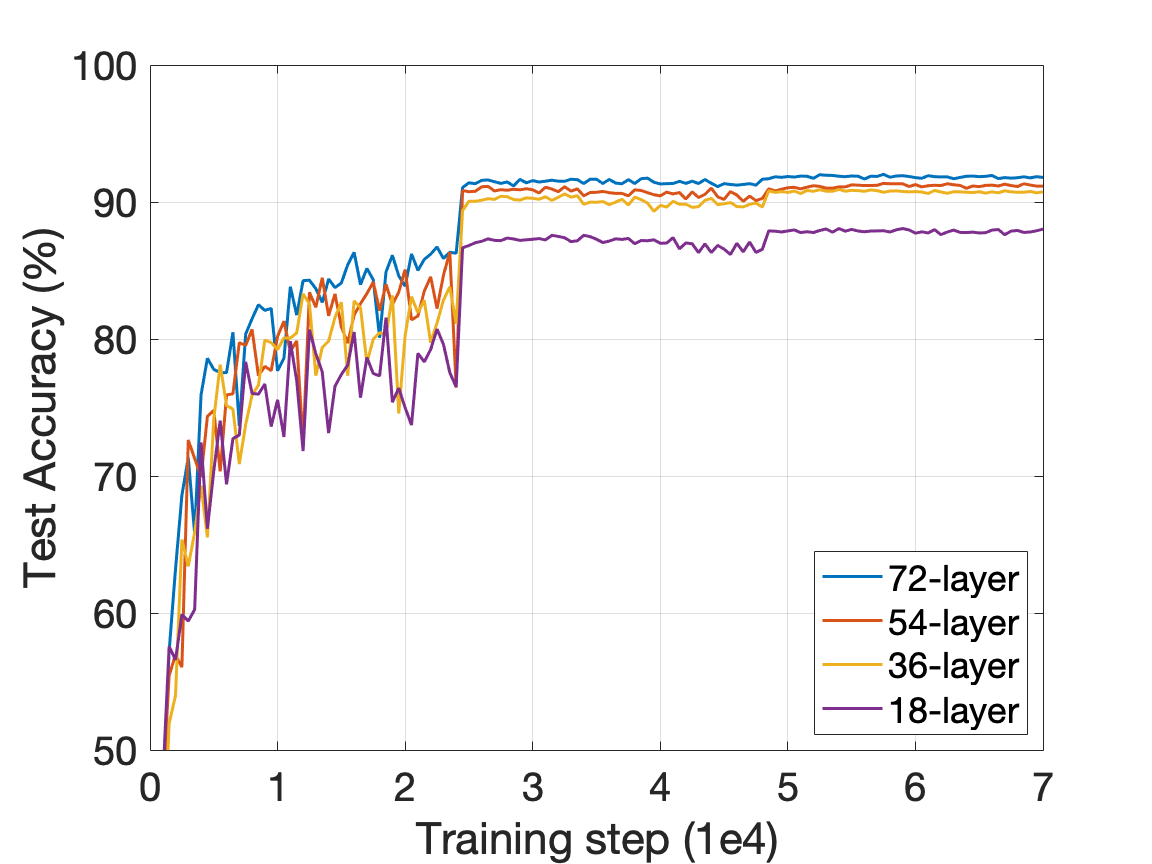}}
  
  \subfloat[ResNet-D on CIFAR-100. \label{fig: train1 cifar100}]{
    \includegraphics[scale=0.35]{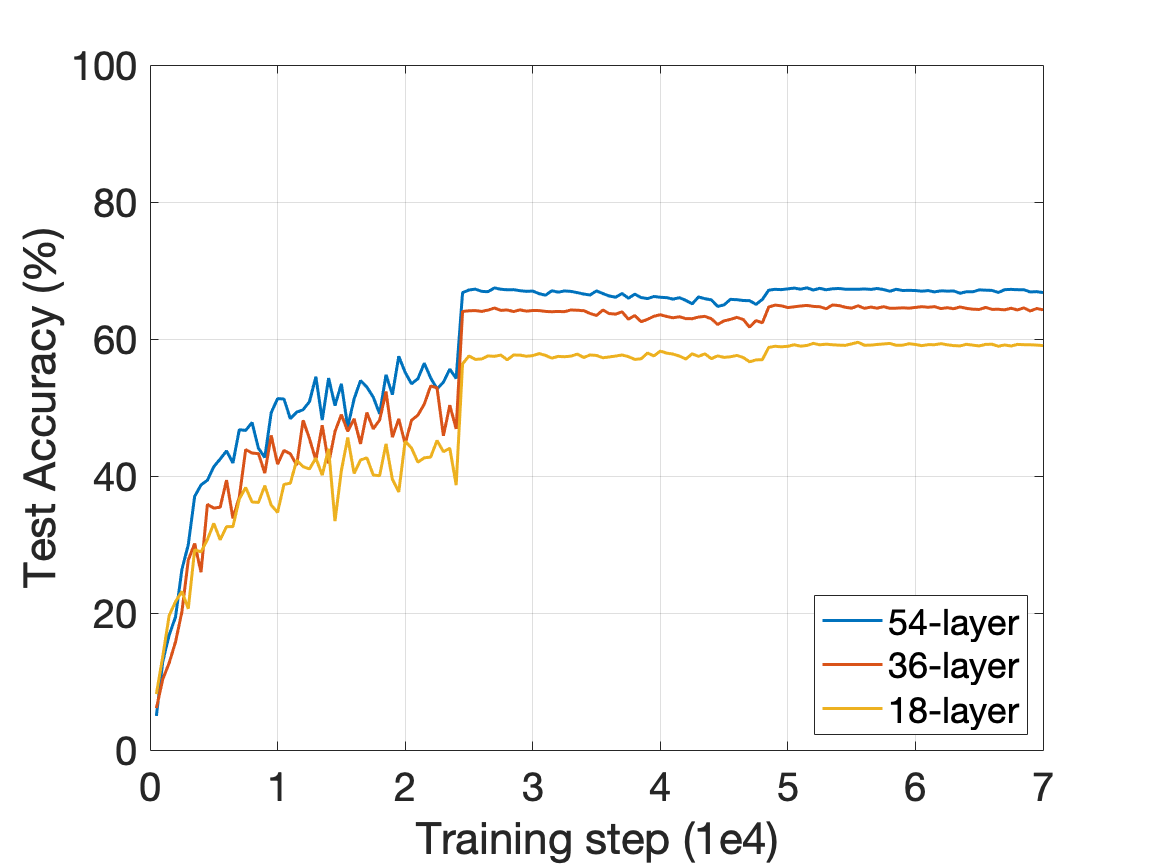}}
 \end{minipage}%
\hfill
\begin{minipage}[c]{0.5\linewidth}
  \centering
        \subfloat[ResNet-S on CIFAR-10. \label{fig: train2 cifar10}]{
   \includegraphics[scale=0.35]{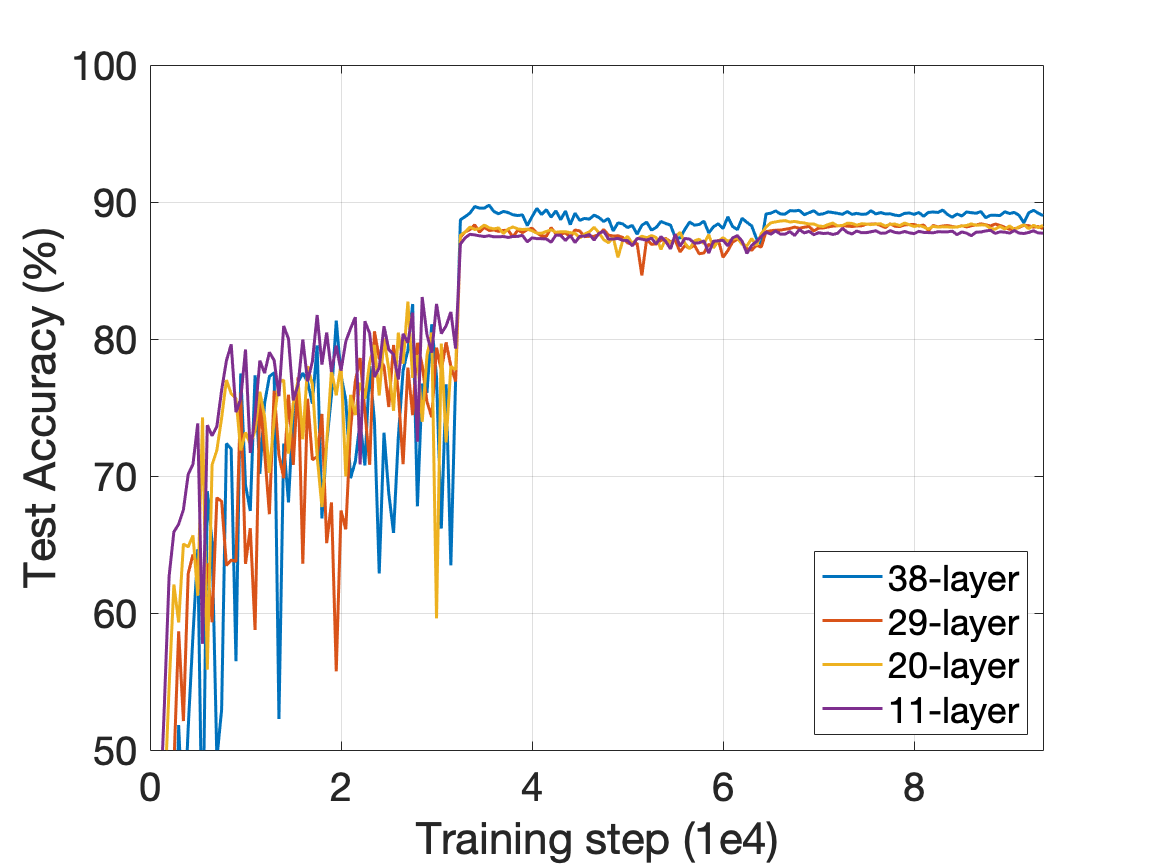} }
   
  \subfloat[ResNet-S on CIFAR-100. \label{fig: train2 cifar100}]{
  \includegraphics[scale=0.35]{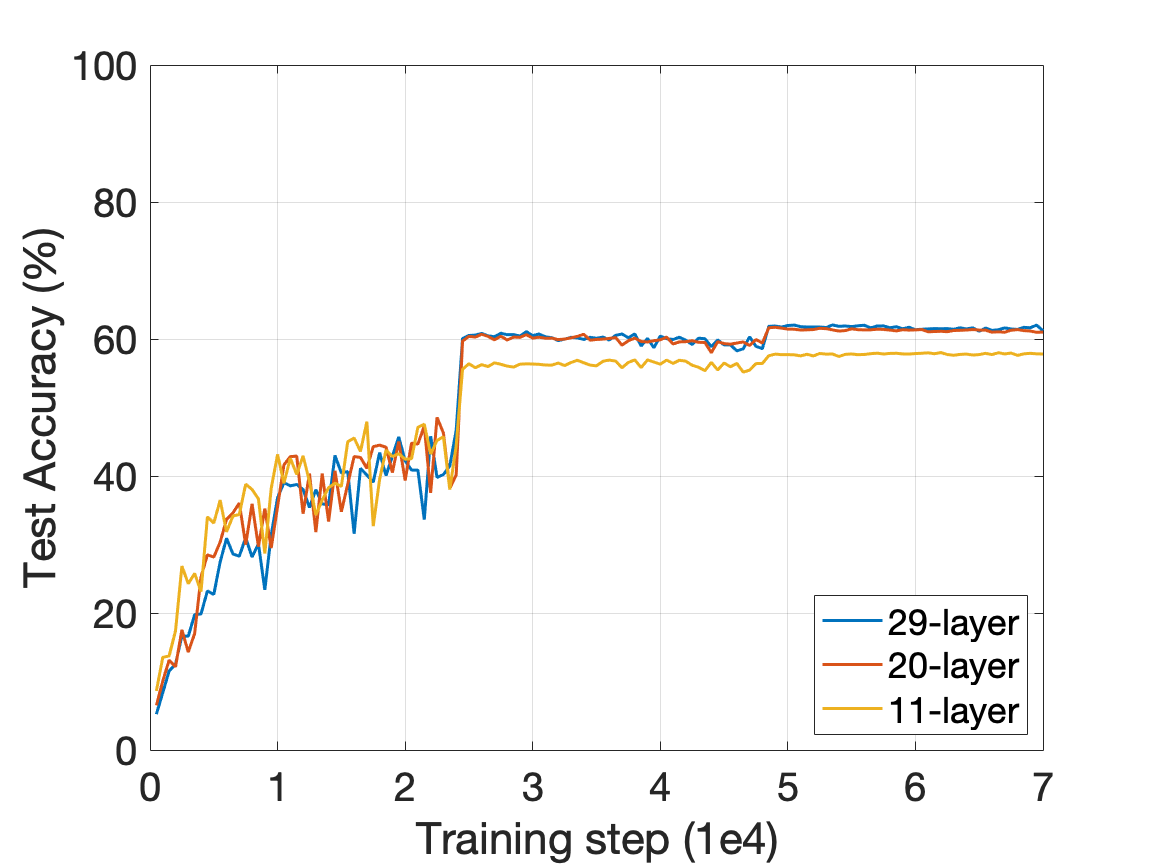}}
   \end{minipage}%
    \caption{Test accuracy of ResNet-D and ResNet-S on CIFAR-10/100. The test accuracy of both ResNet-D and ResNet-S tends to increase as the depth of the network increases.}
  \label{fig: result depth}
\end{figure}

\begin{table}[b!]
  \caption{Average of the test accuracy in the last 5,000 training steps of ResNet-D and ResNet-S on CIFAR-10/100.}
  \label{tab: result depth}
  \centering
  
  \begin{minipage}[c]{0.5\linewidth}
\centering
   \subfloat[ResNet-D on CIFAR-10.\label{tab: form1 cifar10}]{
  \begin{tabular}{|c|c|}  \hline
Depth 	& Test accuracy (\%)  \\ \hline
18		& 87.8509		\\ \hline
36		& 90.7245		\\ \hline
54		& 91.2045		\\ \hline
72		& 91.8091		\\ \hline
  \end{tabular}}
  
  \subfloat[ResNet-D on CIFAR-100.\label{tab: form1 cifar100}]{
  \begin{tabular}{|c|c|}  \hline
Depth 	& Test accuracy (\%)	 \\ \hline
18		& 59.1073		\\ \hline
36		& 64.3582		\\ \hline
54		& 67.0527		\\ \hline
  \end{tabular}}
\end{minipage}%
\hfill
\begin{minipage}[c]{0.5\linewidth}
  \centering
  \subfloat[ResNet-S on CIFAR-10.\label{tab: form2 cifar10}]{
  \begin{tabular}{|c|c|}  \hline
Depth 	& Test accuracy (\%)	 \\ \hline
11		& 87.7782		\\ \hline
20		& 88.1655		\\ \hline
29		& 88.1845		\\ \hline
38		& 89.0927		\\ \hline
  \end{tabular}} 

   \subfloat[ResNet-S on CIFAR-100.\label{tab: form2 cifar100}]{
  \begin{tabular}{|c|c|}  \hline
Depth 	& Test accuracy (\%)	 \\ \hline
11		& 57.8282		\\ \hline
20		& 61.1509		\\ \hline
29		& 61.4891		\\ \hline
  \end{tabular}}
  \end{minipage}

\end{table}

\subsection{Effect of Perturbation on Test Accuracy}

We evaluate the trained networks on images with different types of perturbation. Given a test image $x$, its corrupted image is obtained via $x \mapsto x+\eta$,
where two types of the additive noise $\eta$ are considered:
\begin{subequations}
\begin{align}
\text{unstructured:} \quad &\eta\sim\mathcal{N}(0,\sigma^2), \\
\text{structured:} \quad &\eta=\epsilon x_0, \label{eq: noise structured}
\end{align} \label{eq: noise}
\end{subequations}
where $x_0$ is a fixed image chosen from the test images.

In Table \ref{tab: result noise}, we list the test accuracy of ResNet-D and ResNet-S on perturbed test images, which are evaluated using the learned parameters of the network from the last training step. Note that the networks are trained on the uncorrupted training set. The structured noise $x_0$ used in the experiments is shown in Figure \ref{fig: test image} and is added to the test images. Different values of $\sigma$ and $\epsilon$ are used to vary the noise level of $\eta$. One can observe from Table \ref{tab: noise res1} that when the noise level increases, the test accuracy of ResNet-D decreases. For low levels of perturbation, the accuracy remains high. We observe that deeper networks tend to have higher test accuracies after corruption of the test images.  Similar conclusion can be drawn from Table \ref{tab: noise res2}, in particular, that a deeper ResNet-S seems to be more robust to corrupted test data.

\begin{table}[b!]
  \caption{Test accuracy (\%) with corrupted test images of ResNet-D and ResNet-S on CIFAR-10/100. Each network is trained on the uncorrupted training images of the dataset, and is evaluated using the learned parameters from the last training step on corrupted test images which are obtained via Equation \eqref{eq: noise}.}
  \label{tab: result noise}
  \centering
  \subfloat[ResNet-D. \label{tab: noise res1}]{
  \begin{tabular}{|c|c|c|c|c|c|c|}  \hline
  \multirow{ 2}{*}{Dataset} & \multirow{ 2}{*}{Depth} &    With no noise &    \multicolumn{2}{c|}{With unstructured noise} &    \multicolumn{2}{c|}{With structure noise} \\ \cline{3-7}
&  	& $\sigma=\epsilon=0$	& $\sigma=0.02$	& $\sigma=0.05$ & $\epsilon=0.25$	& $\epsilon=0.75$	\\ \hline
\multirow{ 4}{*}{CIFAR-10} 	& 18	& 88.02	& 83.52	& 50.55 & 83.40	& 41.24\\ \cline{2-7}
&  36	& 90.74	& 85.48	& 56.70 & 84.54	& 35.18\\ \cline{2-7}
&  54	& 91.16	& 86.78	& 63.77 & 85.78	& 32.07\\ \cline{2-7}
& 72	& 91.79	& 86.95	& 61.73 & 86.95	& 40.25 \\ \hline
\multirow{ 3}{*}{CIFAR-100}	& 18	& 59.04	& 46.15	& 16.01 & 55.68	& 29.92 \\ \cline{2-7}
& 36	& 64.27	& 51.36	& 24.08 & 60.53	& 32.93 \\ \cline{2-7}
& 54	& 66.78	& 52.82	& 23.08 & 62.55	& 33.17 \\ \hline
  \end{tabular}
  }

    \subfloat[ResNet-S. \label{tab: noise res2}]{
  \begin{tabular}{|c|c|c|c|c|c|c|}  \hline
  \multirow{ 2}{*}{Dataset} & \multirow{ 2}{*}{Depth} &    With no noise &    \multicolumn{2}{c|}{With unstructured noise} &    \multicolumn{2}{c|}{With structure noise} \\ \cline{3-7}
  &  	& $\sigma=\epsilon=0$	& $\sigma=0.02$	& $\sigma=0.05$ & $\epsilon=0.25$	& $\epsilon=0.75$	\\ \hline
\multirow{ 4}{*}{CIFAR-10} 	& 11  & 87.73	& 82.43	& 52.44 & 82.24	& 33.66\\ \cline{2-7}
& 20	& 88.29	& 83.22	& 56.51  & 82.94	& 34.93\\ \cline{2-7}
& 29	& 88.05	& 83.68	& 55.82  & 83.50	& 36.80\\ \cline{2-7}
& 38	& 89.00	& 85.67	& 59.86 & 83.58	& 34.13\\ \hline
\multirow{ 3}{*}{CIFAR-100}	& 11	& 57.81	& 45.65	& 21.57	& 54.90	& 34.78   \\ \cline{2-7}
& 20	& 61.01	& 47.18	& 21.09  & 57.93	& 37.17\\ \cline{2-7}
& 29	& 61.20	& 54.72	& 31.66 & 58.15	& 35.63\\ \hline
  \end{tabular}
  }
  \end{table}

\begin{figure}[b!]
  \centering
	\subfloat[CIFAR-10. \label{fig: test image 10}]{\includegraphics[scale=0.25]{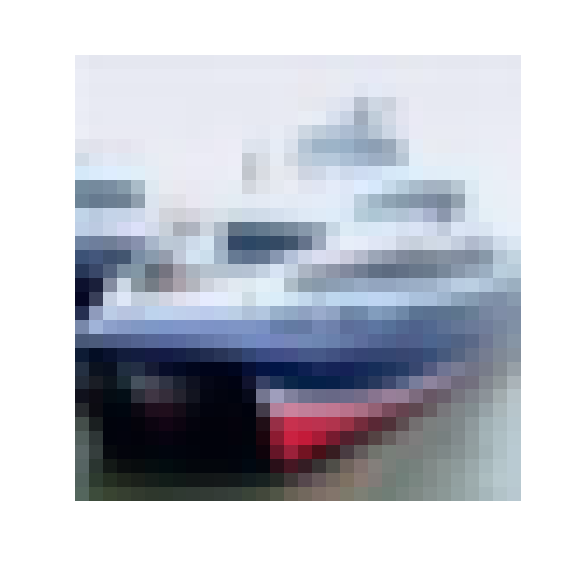}} \quad\quad\quad\quad
	\subfloat[CIFAR-100. \label{fig: test image 100}]{\includegraphics[scale=0.25]{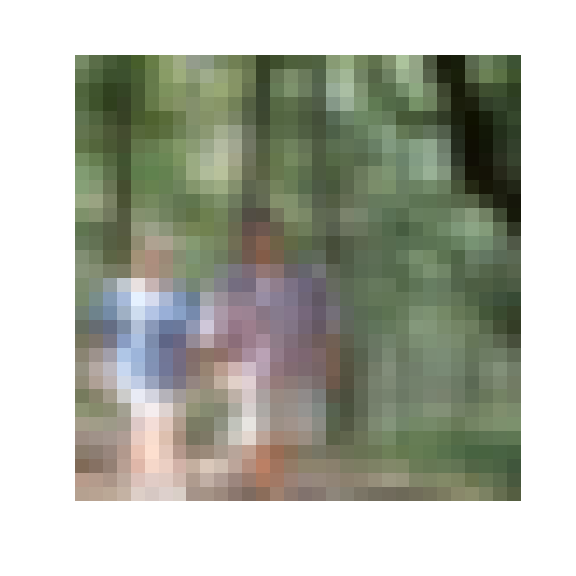}}
	\caption{The structured noise $x_0$ used in the experiments in Table \ref{tab: result noise}. The use of $x_0$ is defined in Equation \eqref{eq: noise structured}. (a) A test image in CIFAR-10 with label ``ship". (b) A test image in CIFAR-100 with label ``forest".}
  \label{fig: test image}
\end{figure}

We illustrate the results in Figures \ref{fig: res1 noise} and \ref{fig: res2 noise} using the trained 36-layer ResNet-D and 20-layer ResNet-S on three test images in CIFAR-10. The test images are labeled as ``bird", ``dog", and ``horse", respectively. In Figures \ref{fig: res1 noise} and \ref{fig: res2 noise}, three test images and the corresponding corrupted images are shown, including the corresponding probability distributions predicted by the trained networks. One observation is that the probability of predicting the true label correctly tends to decrease as the corruption level increases.  For example, consider the case where ResNet-S is applied to the ``horse" image (the last two columns in Figure \ref{fig: res2 noise}). Figure \ref{fig: res2 noise0} shows that the probability that the noise-free image $x$ is a ``horse" is 0.9985. When random noise $\eta$ is added to $x$, {\it i.e.} $x\mapsto x+\eta$ with $\eta\sim\mathcal{N}(0,\sigma^2)$, the probability of correctly predicting $x+\eta$ to be a ``horse" drops to 0.8750 and 0.7940 (for $\sigma$ equal to 0.02 and 0.05, respectively). This is illustrated in Figure \ref{fig: res2 noiseN}.

When the corruption level increases, the label with the second highest predicted probability may change. Take for example ResNet-S on the ``dog" image (the middle two columns in Figure \ref{fig: res2 noise}). Let $x$ be the original ``dog" image. When random noise $\eta$ is added to $x$, the second prediction made by the network changes from a ``cat" (with probability 0.1410) to a ``frog" (with probability 0.1717)  (as $\sigma$ increases from 0.02 to 0.05). This is within the stability bounds from Section~\ref{sec: discrete}. When we perturb a test image by another image (Figure \ref{fig: test image 10}), we observe similar stability results under this structured form of corruption. This is illustrated on the ``bird" image in the first two columns of Figure \ref{fig: res1 noise}.

\begin{figure}[b!]
\centering
\subfloat[Noise-free test images $x$. \label{fig: res1 noise0}]{
\includegraphics[scale=0.1]{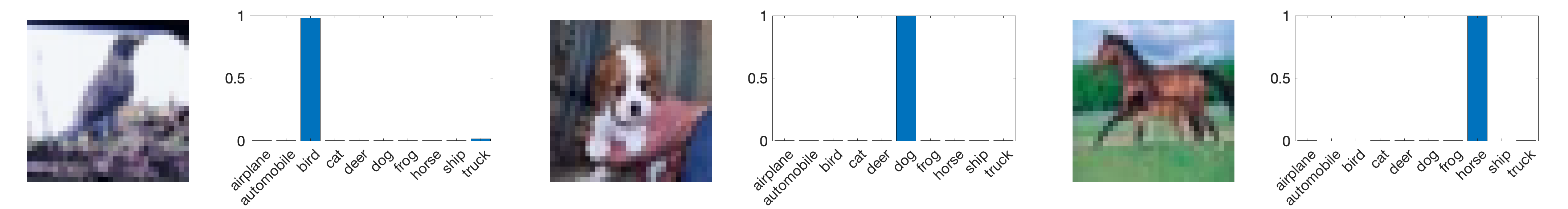} }

\subfloat[$x$ with unstructured noise (first row: $\sigma=0.02$; second row: $\sigma=0.05$). \label{fig: res1 noiseN}]{
\includegraphics[scale=0.1]{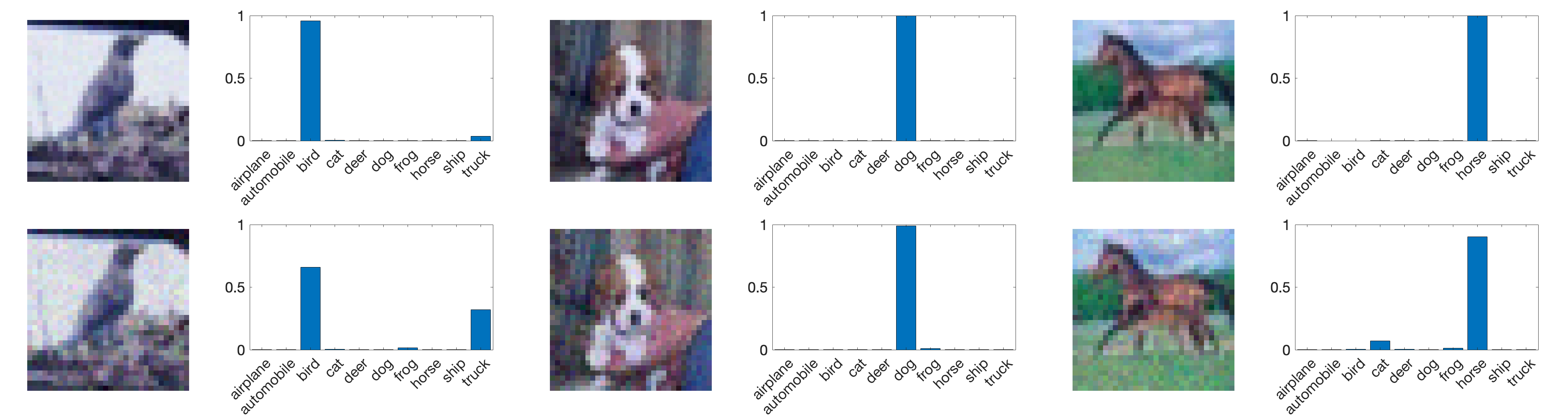} }

\subfloat[$x$ with structured noise (first row: $\epsilon=0.25$; second row: $\epsilon=0.75$). \label{fig: res1 noiseB}]{
\includegraphics[scale=0.1]{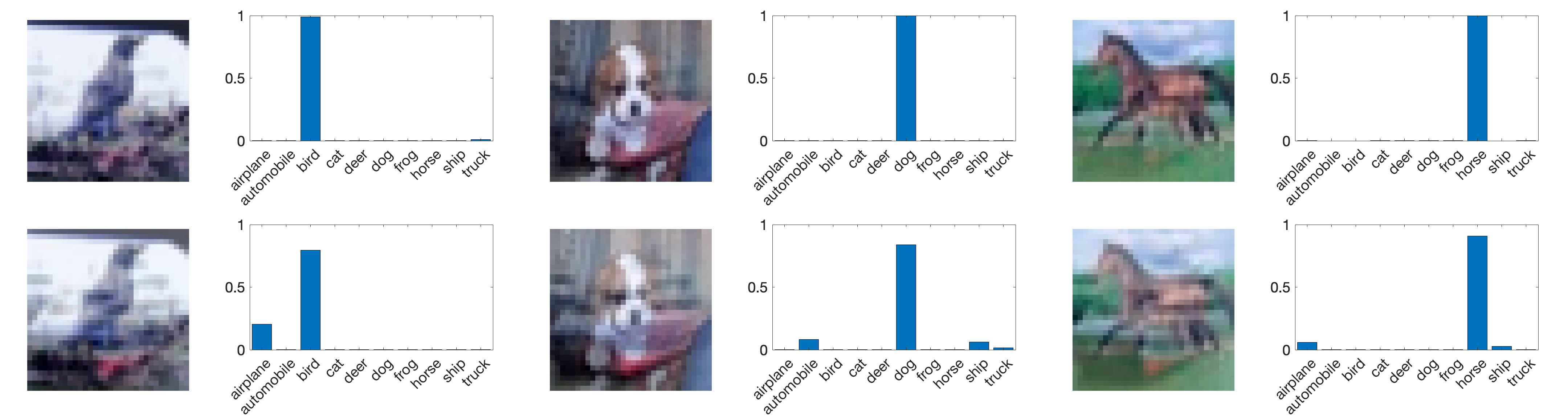} }
 \caption{The trained 36-layer ResNet-D on corrupted test images from CIFAR-10. (a) Three noise-free test images $x$ and  the predicted probability distributions, (b) $x$ with unstructured noise ({\it i.e.} $x+\eta$ with $\eta\sim\mathcal{N}(0,\sigma^2)$) and  the predicted probability distributions, (c) $x$ with structured noise ({\it i.e.} $x+\epsilon x_0$ with $x_0$ shown in Figure \ref{fig: test image 10}) and  the predicted probability distributions. }
\label{fig: res1 noise}
\end{figure}

\begin{figure}[b!]
\centering
\subfloat[Noise-free test images $x$. \label{fig: res2 noise0}]{
\includegraphics[scale=0.1]{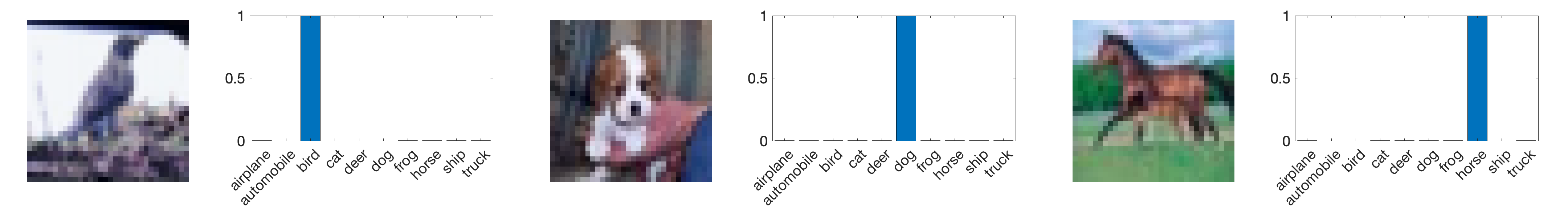} }

\subfloat[$x$ with unstructured noise (first row: $\sigma=0.02$; second row: $\sigma=0.05$). \label{fig: res2 noiseN}]{
\includegraphics[scale=0.1]{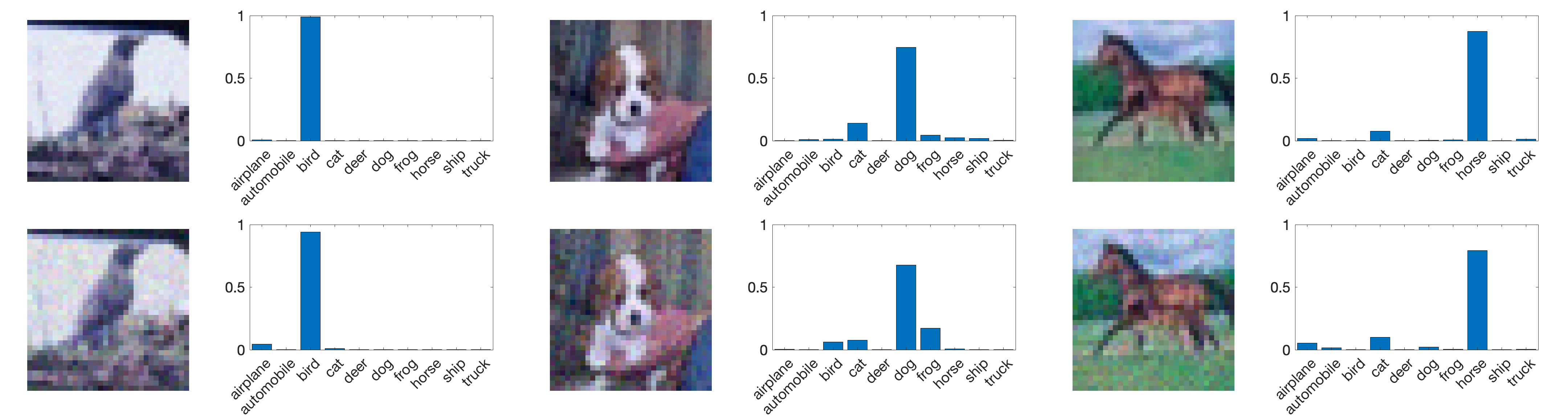} }

\subfloat[$x$ with structured noise (first row: $\epsilon=0.25$; second row: $\epsilon=0.75$). \label{fig: res2 noiseB}]{
\includegraphics[scale=0.1]{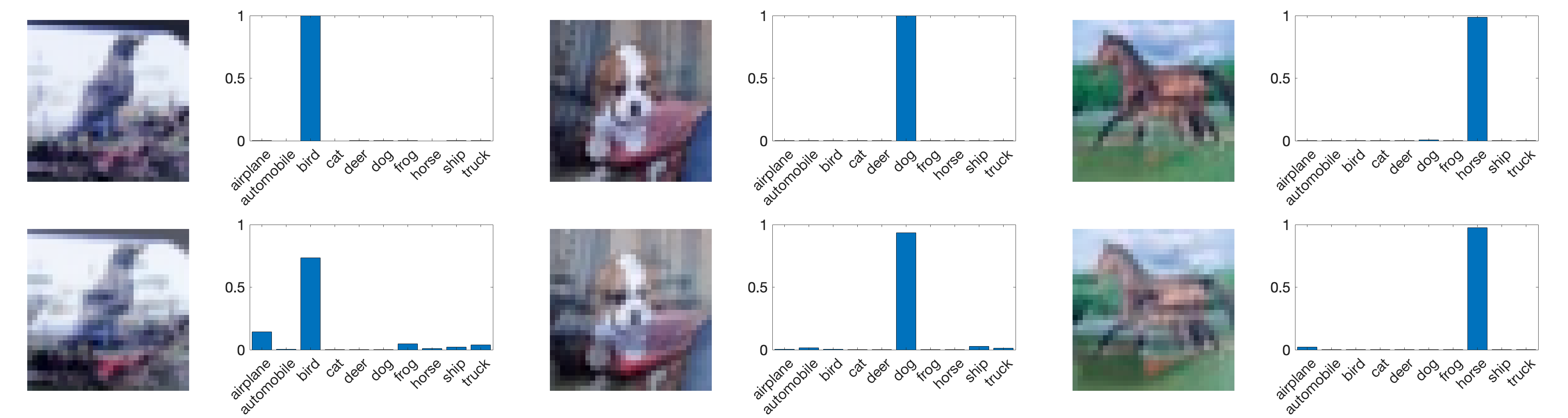} }
 \caption{The trained 20-layer ResNet-S on corrupted test images from CIFAR-10. (a) Three noise-free test images $x$ and  the predicted probability distributions, (b) $x$ with unstructured noise ({\it i.e.} $x+\eta$ with $\eta\sim\mathcal{N}(0,\sigma^2)$) and  the predicted probability distributions, (c) $x$ with structured noise ({\it i.e.} $x+\epsilon x_0$ with $x_0$ shown in Figure \ref{fig: test image 10}) and  the predicted probability distributions. }
\label{fig: res2 noise}
\end{figure}

Equations \eqref{eq: res1 net2} and \eqref{eq: res2 net2} show that perturbation in the output depends on the perturbation in the input and the weight matrices in the network. In theory, if the norm of the additive noise to the input increases, perturbation in the output may be less controllable. Table \ref{tab: result noise} and Figures \ref{fig: res1 noise} and \ref{fig: res2 noise} indicate that changes in the output may affect test accuracy.

\section{Discussion}
We have provided a relationship between ResNet (or other networks with skip-connections) to an optimal control problem with differential inclusions and used this connection to gain some insights to the behavior of the network. We have shown that the system is well-posed and have provided growth bounds on the features. The continuous-time analysis is helpful in interpreting the success of networks with skip-connections. For example, since the forward flow of well-posed dynamical systems will have regular paths between inputs and outputs, we should expect a similar result for very deep networks. This is likely a reason why DNNs with skip-connections generalize well, since similar inputs should follow similar paths and  the skip-connections make the paths more regular. 

In practice, ResNet and other DNNs have additional layers which are not currently captured by the optimal control formulation (for example, normalization and pooling). In this setting, we provided stability bounds for the entire network as a function of each of the layers' learnable parameters. In some cases, the network is stable regardless of its depth due to structural constraints or regularization. The constraints may also smooth the energy landscape so that the minimizers are flatter, which will be considered in future work.

It is also worth noting that ResNet and other DNNs are often ``stabilized" by other operations. From experiments, one can observe that batch normalization has the additional benefit of controlling the norms of the features during forward propagation. Without batch normalization and without strong enough regularization, the features will grow unboundedly in the residual blocks. It would be interesting to analyze the role of different stabilizers in the network on the network's ability to generalize to new data.

\section*{Acknowledgement}
The authors acknowledge the support of AFOSR, FA9550-17-1-0125 and the support of NSF CAREER grant $\#1752116$. 

\appendix

\section{Proofs of the Main Results} \label{sec: proof}

We provide the proofs of the results presented in this work.

\begin{proof}[\bf Proof of Proposition \ref{prop: pooling}]
Let $x$ be a feature in $\R^{h\times w\times d}$ and define $y$ by Equation \eqref{eq: def 2d pooling}. Observe from Equations \eqref{eq: def conv} and \eqref{eq: def 2d pooling} that:
\begin{align*}
(y_{i,j,k})^2&\le\begin{cases}
(x_{2i-1,2j-1,k})^2 + (x_{2i-1,2j,k})^2 + (x_{2i,2j-1,k})^2 + (x_{2i,2j,k})^2, &\,\text{if } i\le h/2 \text{ and }j\le w/2, \\
(x_{2i-1,2j-1,k})^2 + (x_{2i-1,2j,k})^2 , &\,\text{if } i> h/2 \text{ and }j\le w/2, \\
(x_{2i-1,2j-1,k})^2 + (x_{2i,2j-1,k})^2, &\,\text{if } i\le h/2 \text{ and }j> w/2, \\
(x_{2i-1,2j-1,k})^2, &\,\text{if } i> h/2 \text{ and }j> w/2,
\end{cases} \\
|y_{i,j,k}|&\le\begin{cases}
\max\left\{|x_{2i-1,2j-1,k}|,  |x_{2i-1,2j,k}|,  |x_{2i,2j-1,k}|,  |x_{2i,2j,k}|\right\}, &\quad\text{if } i\le h/2 \text{ and }j\le w/2, \\
\max\left\{|x_{2i-1,2j-1,k}|, |x_{2i-1,2j,k}| \right\}, &\quad\text{if } i> h/2 \text{ and }j\le w/2, \\
\max\left\{|x_{2i-1,2j-1,k}| + |x_{2i,2j-1,k}| \right\}, &\quad\text{if } i\le h/2 \text{ and }j> w/2, \\
|x_{2i-1,2j-1,k}|, &\quad\text{if } i> h/2 \text{ and }j> w/2,
\end{cases}
\end{align*}
for all $i=1,2,\dots,h$, $j=1,2,\dots,w$, and $k=1,2,\dots,d$, and thus:
\begin{subequations}
\begin{align}
\|y\|_F^2 &\le \|x\|_F^2, \\
\max_{i\in[h_1],\,j\in[w_1],\,k\in[d]}|y_{i,j,k}| &\le \max_{i\in[h],\,j\in[w],\,k\in[d]}|x_{i,j,k}|,
\end{align} \label{eq: prop pooling proof}
\end{subequations}
where $h_1:=\ceil*{h/2}$ and $w_1=\ceil*{w/2}$. Thus, by Equations \eqref{eq: prop vec}, \eqref{eq: def 2d pooling vec} and \eqref{eq: prop pooling proof}, we have:
\begin{align*}
\|P_2(\vect(x))\|_{\ell^2(\R^{h_1w_1d})}^2 &= \|\vect(y)\|_{\ell^2(\R^{h_1w_1d})}^2 = \|y\|_F^2 \le \|x\|_F^2 = \|\vect(x)\|_{\ell^2(\R^{hwd})}^2, \\
\|P_2(\vect(x))\|_{\ell^\infty(\R^{h_1w_1d})} &= \|\vect(y)\|_{\ell^\infty(\R^{h_1w_1d})} = \max_{i\in[h_1],\,j\in[w_1],\,k\in[d]}|y_{i,j,k}| \\
&\le \max_{i\in[h],\,j\in[w],\,k\in[d]}|x_{i,j,k}| = \|\vect(x)\|_{\ell^\infty(\R^{hwd})}.
\end{align*} 
This proves Equation \eqref{eq: prop 2d pooling}. Equation \eqref{eq: prop global pooling} can be derived using a similar argument as above.
\end{proof}

\begin{proof}[\bf Proof of Theorem \ref{thm: inclusion resnet}]
Take $(H, \|\cdot\|)=\left(\R^d,\|\cdot\|_{\ell^2(\R^d)}\right)$,  $I=[0,\infty)$, 
\begin{align}
F(t,x(t)) := A_2(t)\,\sigma(A_1(t)x(t)+b_1(t)) - b_2(t), \label{eq: inclusion explicit}
\end{align}
and $C$ be the multi-valued mapping such that $C(t)=\R^d_+$ for all $t\in[0,T]$. We will prove that conditions (i)-(iv) in Theorem \ref{thm: inclusion} are satisfied. Without ambiguity, we write $\|\cdot\|_2$ for $\|\cdot\|_{\ell^2(\R^d)}$.

\begin{enumerate}[label*=(\roman*)]
\item For each $t\in [0,T]$, it is clear that $C(t)$ is a nonempty closed subset of $H$, and by \cite{poliquin1996prox}, $C(t)$ is $r$-prox-regular.
\item Setting $v(t)=0$ for all $t\in [0,T]$ yields Equation \eqref{eq: inclusion cond2}.
\item Let $x,y:I\to \R^d$. By Equation \eqref{eq: inclusion explicit} and the assumptions that $\|A_1(t)\|_2 \, \|A_2(t)\|_2\le c$ for all $t>0$ and that $\sigma$ is contractive, we have:
\begin{align*}
\| F(t,x(t))-F(t,y(t)) \|_2 &= \|A_2(t)\ \sigma(A_1(t)x(t)+b_1(t))  - A_2(t)\ \sigma(A_1(t)y(t)+b_1(t)) \|_2 \\
&\le \|A_2(t)\|_2\ \|\sigma(A_1(t)x(t)+b_1(t))-\sigma(A_1(t)y(t)+b_1(t))\|_2 \\
&\le \|A_2(t)\|_2 \ \|A_1(t)x(t)-A_1(t)y(t)\|_2 \\
&\le \|A_2(t)\|_2 \ \|A_1(t)\|_{2} \ \|x(t)-y(t)\|_2 \\
&\le c \|x(t)-y(t)\|_2.
\end{align*}
\item Let $x:[0,T]\to \R^d$. A similar derivation as above yields:
\begin{align*}
\| F(t,x(t))\|_2 &= \|A_2(t)\ \sigma(A_1(t)x(t)+b_1(t)) - b_2(t)\|_2 \\
&\le \|A_2(t)\|_2 \ \|\sigma(A_1(t)x(t)+b_1(t)) \|_2 + \|b_2(t)\|_2 \\
&\le \|A_2(t)\|_2 \ \|A_1(t)x(t)+b_1(t) \|_2 + \|b_2(t)\|_2 \\
&\le \|A_2(t)\|_2\  \left(\|A_1(t)\|_{2}\|x(t)\|_2+\|b_1(t)\|_2\right) + \|b_2(t)\|_2 \\
&\le \beta(t)\left(1+\|x(t)\|_2\right),
\end{align*}
where 
\begin{align*}
\beta(t)&:= \max\left\{c,\|A_2(t)\|_2\,\|b_1(t)\|_2 + \|b_2(t)\|_2\right\}.
\end{align*} 
\end{enumerate}
Therefore, by Theorem \ref{thm: inclusion}, there exists a unique absolutely continuous solution $x$ to Equation~\eqref{eq: sweeping process} for almost every $x_0\in\R^d_+$. In particular, by Remark \ref{rem: inclusion}, the solution $x$ satisfies that $x(t)\in\R^d_+$ for all $t>0$.
\end{proof}

\begin{proof} [\bf Proof of Theorem \ref{prop: gen stable}]
Fix $t>0$. Taking the inner product of Equation~\eqref{eq: general form continuous} with $x$ yields:
\begin{align*}
x(t)^T\dfrac{\text{d}}{\text{dt}}x(t) +  x(t)^TA_2(t)\,\sigma(A_1(t)x(t)+b_1(t)) - x(t)^Tb_2(t) = x(t)^Tp_x(t)
\end{align*}
for some $p_x(t) \in  -\partial{I}_{\R^{d}_+}(x)$. Note that $0\in\R^{d}_+$ and $0 \in \partial{I}_{\R^{d}_+}(x)$. Thus, by monotonicity of the sub-differential, we have:
\begin{align*}
x(t)^Tp_x(t) = (x(t)-0)^T(p_x(t)-0)\leq 0,
\end{align*}
which implies that:
\begin{align*}
x(t)^T\dfrac{\text{d}}{\text{dt}}x(t) +  x(t)^TA_2(t)\,\sigma(A_1(t)x(t)+b_1(t)) - x(t)^Tb_2(t) \leq 0.
\end{align*}
Therefore, after re-arranging terms, we have:
\begin{align*}
\dfrac{\text{d}}{\text{dt}} \left(\dfrac{\|x(t)\|_2^2}{2}\right) &= x(t)^T\dfrac{\text{d}}{\text{dt}}x(t) \le - x(t)^TA_2(t)\,\sigma(A_1(t)x(t)+b_1(t)) + x(t)^Tb_2(t).
\end{align*}
By Theorem \ref{thm: inclusion resnet}, $x(t) \in \R_+^d$ for a.e. $t>0$, and thus the inner product $x(t)^Tb_2(t)$ is bounded above by the positive part of $b_2(t)$; that is,
\begin{align*}
x(t)^Tb_2(t) \le x(t)^T\left(b_2(t)\right)_+ \le \|x(t)\|_2 \, \|(b_2(t))_+\|_2.
\end{align*}
Therefore, by the assumption that $\sigma$ is contractive and $\sigma(0)=0$, we have:
\begin{align*}
\dfrac{\text{d}}{\text{dt}} \left(\dfrac{\|x(t)\|_2^2}{2}\right) &\le \|A_2(t)\|_2 \  \|x(t)\|_2\  \|\sigma(A_1(t)x(t)+b_1(t))\|_2 +  \|x(t)\|_2 \  \|(b_2(t))_+\|_2\\
  &\le \|A_2(t)\|_2 \  \|x(t)\|_2\  \|A_1(t)x(t)+b_1(t)\|_2 +  \|x(t)\|_2 \  \|(b_2(t))_+\|_2\\
  &\le \|A_1(t)\|_2 \  \|A_2(t)\|_2 \  \|x(t)\|_2^2+\left(\|A_2(t)\|_2  \|b_1(t)\|_2+\|(b_2(t))_+\|_2 \right)  \|x(t)\|_2.
\end{align*}
Applying Theorem \ref{thm: gronwall} with $u=\|x\|_2^2/2$, $f=2\|A_1\|_2 \, \|A_2\|_2$, $g=\sqrt{2}\left(\|A_2\|_2 \|b_1\|_2+\|(b_2)_+\|_2\right)$, $c=\|x(0)\|_2^2/2$, $t_0=0$, and $\alpha=1/2$ yields:
\begin{align*}
\|x(t)\|_2 &\le \|x(0)\|_2 \, \exp\left(\int_0^t \|A_1(s)\|_2 \, \|A_2(s)\|_2 \, \dd{s}\right)\\
&\quad + \int_0^t \left(\|A_2(s)\|_2 \, \|b_1(s)\|_2+\|\left(b_2(s)\right)_+\|_2 \right) \, \exp\left(\int_s^t \|A_1(r)\|_2 \, \|A_2(r)\|_2 \, \dd{r}\right) \, \dd{s},
\end{align*}
which proves Equation \eqref{eq: general bound 1}. 

Next, let $x$ and $y$ be the unique absolutely continuous  solutions to Equation~\eqref{eq: general form continuous}, with different initial values $x(0)$ and $y(0)$. Then:
\begin{align*}
&\left(\dfrac{\text{d}}{\text{dt}}x(t)-\dfrac{\text{d}}{\text{dt}}y(t)\right) +  A_2(t)\,\left(\sigma(A_1(t)x(t)+b_1(t)) - \sigma(A_1(t)y(t)+b_1(t))\right)  =p_x(t)-p_y(t)
\end{align*}
for some $p_x(t) \in  -\partial{I}_{\R^{d}_+}(x)$ and $p_y(t) \in  -\partial{I}_{\R^{d}_+}(y)$. By monotonicity of the subdifferentials, we have:
\begin{align*}
(x(t)-y(t))^T(p_x(t)-p_y(t))\leq 0,
\end{align*}
which implies that:
\begin{align*}
\dfrac{\text{d}}{\text{dt}} \left(\dfrac{\|x(t)-y(t)\|_2^2}{2}\right) &=(x(t)-y(t))^T\left(\dfrac{\text{d}}{\text{dt}}x(t)-\dfrac{\text{d}}{\text{dt}}{y}(t)\right)\\
 &\leq -  (x(t)-y(t))^TA_2(t)\,\left(\sigma(A_1(t)x(t)+b_1(t)) - \sigma(A_1(t)y(t)+b_1(t))\right)\\
  &\leq \|A_2(t)\|_2 \, \|x(t)-y(t)\|_2\, \|\sigma(A_1(t)x(t)+b_1(t)) - \sigma(A_1(t)y(t)+b_1(t))\|_2.
    \end{align*}
Therefore, by the assumption that $\sigma$ is contractive and $\sigma(0)=0$, we have:
\begin{align*}
\dfrac{\text{d}}{\text{dt}} \left(\dfrac{\|x(t)-y(t)\|_2^2}{2}\right) &\leq \|A_2(t)\|_2 \, \|x(t)-y(t)\|_2\, \|A_1(t)(x(t) -y(t))\|_2\\
    &\leq \|A_1(t)\|_2 \, \|A_2(t)\|_2 \, \|x(t)-y(t)\|^2_2.
    \end{align*}
Applying Theorem \ref{thm: gronwall} with $u=\|x\|_2^2/2$, $f=2\|A_1\|_2 \, \|A_2\|_2$, $g=0$, $c=\|x(0)\|_2^2/2$, $t_0=0$, and $\alpha=1$ yields:
\begin{align*}
\|x(t)-y(t)\|_2 \leq \|x(0)-y(0)\|_2\,  \exp\left(\int_0^t \|A_1(s)\|_2 \, \|A_2(s)\|_2 \, \dd{s}\right),
    \end{align*}
which proves Equation \eqref{eq: general bound 2}.
\end{proof}

\begin{proof} [\bf Proof of Theorem \ref{prop: form1 stable}]
Taking the inner product of Equation~\eqref{eq:inclusion_form1} with $x$ yields:
\begin{align*}
x(t)^T\dfrac{\text{d}}{\text{dt}}x(t) +  x^TA_2(t)\sigma(A_1(t)x(t)+b_1(t)) - x^Tb_2(t) = x(t)^Tp_x(t)
\end{align*}
for some $p_x(t) \in  -\partial{I}_{\R^{d}_+}(x)$.  Using the same argument as in the proof of Theorem \ref{prop: gen stable}, we have:
\begin{align*}
\dfrac{\text{d}}{\text{dt}} \left(\dfrac{\|x(t)\|_2^2}{2}\right) = x(t)^T\dfrac{\text{d}}{\text{dt}}x(t) \le - x(t)^TA_2(t)\,\sigma(A_1(t)x(t)+b_1(t)) + x(t)^T(b_2(t))_+.
\end{align*}
By Remark \ref{rem: inclusion}, $x \in \R_+^d$, and by assumption, $A_2(t)\ge0$. Thus:
\begin{align*}
x(t)^TA_2(t)\,\sigma(A_1(t)x(t)+b_1(t))\ge0,
\end{align*}
which implies that
\begin{align*}
\dfrac{\text{d}}{\text{dt}} \left(\dfrac{\|x(t)\|_2^2}{2}\right)  \le x(t)^T\left(b_2(t)\right)_+ \le \|x(t)\|_2\|(b_2(t))_+\|_2.
\end{align*}
Applying Theorem \ref{thm: gronwall} with $u=\|x\|_2^2/2$, $f=0$, $g=\sqrt{2}\|(b_2)_+\|_2$, $c=\|x(0)\|_2^2/2$, $t_0=0$, and $\alpha=1/2$ yields:
\begin{align*}
\|x(t)\|_2 \le \|x(0)\|_2 + \int_0^t\|\left(b_2(s)\right)_+\|_2\,\dd s,
\end{align*}
which proves Equation \eqref{eq: improved 1}.
\end{proof}

\begin{proof} [\bf Proof of Theorem \ref{prop: form2 stable}]
Taking the inner product of Equation~\eqref{eq:inclusion_form2} with $x$ yields:
\begin{align*}
x(t)^T\dfrac{\text{d}}{\text{dt}}x(t) +  (A_1(t) x(t))^TA_2(t)\,\sigma(A_1(t)x(t)+b_1(t)) - x(t)^Tb_2(t) = x(t)^Tp_x(t)
\end{align*}
for some $p_x(t) \in  -\partial{I}_{\R^{d}_+}(x)$. Using the same argument as in the proof of Theorem \ref{prop: gen stable}, we have:
\begin{align*}
\dfrac{\text{d}}{\text{dt}} \left(\dfrac{\|x(t)\|_2^2}{2}\right)
& \le - (A(t) x(t))^T\,\sigma(A(t)x(t)+b_1(t)) + x(t)^Tb_2(t).
\end{align*}
Define $G:\R^d\to\R^d$ by:
\begin{align*}
G(x):=\sigma(x +b_1(t)).
\end{align*}
Since $\sigma$ ({\it i.e.} ReLU) is monotone, we have that $G$ is also monotone. Thus:
\begin{align*}
\dfrac{\text{d}}{\text{dt}} \left(\dfrac{\|x(t)\|_2^2}{2}\right)
&\leq - (A(t) x(t))^T\,\sigma(A(t)x(t)+b_1(t)) + x(t)^Tb_2(t)\\
&= - (A(t) x(t)-0)^T\left(G(A(t)x(t))-G(0)\right) - (A(t) x(t))^TG(0)+ x(t)^Tb_2(t)\\
&\leq - (A(t) x(t))^T\sigma(b_1(t))+ x(t)^T(b_2(t) \\
&\leq \|x(t)\|_2 \, \left\| \left(-A(t)^T\sigma(b_1(t)) + b_2(t) \right)_+\right\|_2.
\end{align*}
Applying Theorem \ref{thm: gronwall} with $u=\|x\|_2^2/2$, $f=0$, $g=\sqrt{2} \left(-A^T\sigma(b_1) + b_2 \right)_+$, $c=\|x(0)\|_2^2/2$, $t_0=0$, and $\alpha=1/2$ yields:
\begin{align*}
\|x(t)\|_2 \le \|x(0)\|_2 + \int_0^t \left\| \left(-A(s)^T\sigma(b_1(s)) + b_2(s) \right)_+\right\|_2\,\dd s,
\end{align*}
which proves Equation \eqref{eq: improved 2.1}.

Next, let $x$ and $y$ be the unique absolutely continuous  solutions to Equation~\eqref{eq: general form continuous}, with different initial values $x(0)$ and $y(0)$. Using the same argument as in the proof of Theorem \ref{prop: gen stable} yields:
\begin{align*}
\dfrac{\text{d}}{\text{dt}} \left(\dfrac{\|x(t)-y(t)\|_2^2}{2}\right) &=(x(t)-y(t))^T\left(\dfrac{\text{d}}{\text{dt}}x(t)-\dfrac{\text{d}}{\text{dt}}y(t)\right)\\
 &\leq -  (A(t)x(t)-A(t)y(t))^T\left(\sigma(A(t)x(t)+b_1(t)) - \sigma(A(t)y(t)+b_1(t))\right) \le 0,
\end{align*}
where the last inequality is due to monotonicity of $G$. This proves Equation \eqref{eq: improved 2.2}.
\end{proof}

\begin{proof}[\bf Proof of Theorem \ref{thm: res1 net}, part 1]

We will show that 
\begin{align*}
\|x^{n+1}\|_\infty \le \|x^n\|_\infty + c_n
\end{align*}
for all $n=1,2,\dots,3m+3$, where $c_n\ge0$ is independent of the $x^n$.

For the convolution layer (Layer 0), we show that:
\begin{align}
\|x^{1}\|_{\ell^\infty(\R^{h_1w_1d_1})} &\le \|x^0\|_{\ell^\infty(\R^{h_1w_1d_0})} + \|b^0\|_{\ell^\infty(\R^{h_1w_1d_1})} \label{eq: thm res1 conv}
\end{align}
provided that $\|A^0\|_{\ell^\infty(\R^{h_1w_1d_0})\to\ell^\infty(\R^{h_1w_1d_1})}\le1$. By Equation \eqref{eq: layer conv}, we have:
\begin{align*}
\|x^{1}\|_{\ell^\infty(\R^{h_1w_1d_1})} &\le \|A^0x^0\|_{\ell^\infty(\R^{h_1w_1d_1})} + \|b^0\|_{\ell^\infty(\R^{h_1w_1d_1})} \\
&\le \|A^0\|_{\ell^\infty(\R^{h_1w_1d_0})\to\ell^\infty(\R^{h_1w_1d_1})}\|x^0\|_{\ell^\infty(\R^{h_1w_1d_0})} + \|b^0\|_{\ell^\infty(\R^{h_1w_1d_1})} \\
&\le \|x^0\|_{\ell^\infty(\R^{h_1w_1d_0})} + \|b^0\|_{\ell^\infty(\R^{h_1w_1d_1})}.
\end{align*}

For the first stack of ResNet layers (Layer $n$ with $n=1,2\dots,m$), we show that:
\begin{align}
\|x^{n+1}\|_{\ell^\infty(\R^{h_1w_1d_1})} &\le \|x^n\|_{\ell^\infty(\R^{h_1w_1d_1})} + \|b^n_2\|_{\ell^\infty(\R^{h_1w_1d_1})}. \label{eq: thm res1 resnet}
\end{align}
Fix $i\in[h_1w_1d_1]$. By Equation \eqref{eq: layer resnet2}, we have:
\begin{align*}
0\le x^{n+1}_i = (x^n_i-a^n_i(A^n_1x^n+b^n_1)_++(b^n_2)_i)_+,
\end{align*}
where $a^n_i$ denotes the $i$-th row of $A^n_2$ and $(b^n_2)_i$ denotes the $i$-th element of $b^n_2$. Consider two cases. If $x^n_i-a^n_i(A^n_1x^n+b^n_1)_++(b^n_2)_i<0$, then $x^{n+1}_i=0$. Otherwise, since $a^n_i\ge0$ component-wise, it holds that
\begin{align*}
0\le x^{n+1}_i =x^n_i-a^n_i(A^n_1x^n+b^n_1)_++(b^n_2)_i \le x^n_i +(b^n_2)_i.
\end{align*}
Therefore,
\begin{align*}
\|x^{n+1}\|_{\ell^\infty(\R^{h_1w_1d_1})} &\le \|x^n\|_{\ell^\infty(\R^{h_1w_1d_1})} + \|b^n_2\|_{\ell^\infty(\R^{h_1w_1d_1})}.
\end{align*}
Analysis for the remaining ResNet layers, Layers $m+2$ to $2m$ and Layers $2m+2$ to $3m$, is the same.

For the first 2d pooling layer (Layer $n$ with $n=m+1$), we show that:
\begin{align}
\|x^{n+1}\|_{\ell^\infty(\R^{h_2w_2d_2})} \le \|x^n\|_{\ell^\infty(\R^{h_1w_1d_1})}. \label{eq: thm res1 2d pool}
\end{align}
Observe from Figures \ref{fig: network baseline} and \ref{fig: network layers} that $x^j\ge0$ component-wise for all $j=2,3,\dots,3m+2$. Since both $E(P_2(x^n))$ and $\left((A^n)_{|{s=2}}\,x^n+b^n\right)_+$ are component-wise nonnegative, by Equation \eqref{eq: layer 2d pool}, we have the following component-wise inequality:
\begin{align*}
x^{n+1} \le E(P_2(x^n)),
\end{align*}
and thus by Equations \eqref{eq: prop 2d pooling infty} and \eqref{eq: prop relu}:
\begin{align*}
\|x^{n+1}\|_{\ell^\infty(\R^{h_2w_2d_2})} \le \|E(P_2(x^n))\|_{\ell^\infty(\R^{h_2w_2d_2})} = \|P_2(x^n)\|_{\ell^\infty(\R^{h_2w_2d_1})} \le \|x^n\|_{\ell^\infty(\R^{h_1w_1d_1})}.
\end{align*}
Analysis for the second 2D pooling layer, Layer $2m+1$, is the same.

For the global pooling layer (Layer $n$ with $n=3m+1$), we show that:
\begin{align}
\|x^{n+1}\|_{\ell^\infty(\R^{d_3})} \le \|x^n\|_{\ell^\infty(\R^{h_3w_3d_3})}. \label{eq: thm res1 global pool}
\end{align}
By Equations \eqref{eq: prop global pooling infty} and \eqref{eq: prop relu}, the functions $P_g$ and $\relu$ are non-expansive in $\ell^\infty$, and thus:
\begin{align*}
\|x^{n+1}\|_{\ell^\infty(\R^{d_3})} = \|P_g\left((x^n)_+\right)\|_{\ell^\infty(\R^{d_3})} \le \|(x^n)_+\|_{\ell^\infty(\R^{h_3w_3d_3})}  \le \|x^n\|_{\ell^\infty(\R^{h_3w_3d_3})}.
\end{align*}

For the fully connected layer (Layer $n=N-1$ with $N=3m+3$), we show that:
\begin{align}
\|x^{N}\|_{\ell^\infty(\R^{C})} &\le \|x^{N-1}\|_{\ell^\infty(\R^{d_3})} + \|b^{N-1}\|_{\ell^\infty(\R^{C})} \label{eq: thm res1 dense}
\end{align}
provided that $\|W^{N-1}\|_{\ell^\infty(\R^{d_3})\to\ell^\infty(\R^{C})}\le1$. Analysis for the fully connected layer is the same as the analysis for the convolution layer. By Equation \eqref{eq: layer dense}, we have:
\begin{align*}
\|x^{N}\|_{\ell^\infty(\R^{C})} &\le \|W^{N-1}x^{N-1}\|_{\ell^\infty(\R^{C})} + \|b^{N-1}\|_{\ell^\infty(\R^{C})} \\
&\le \|W^{N-1}\|_{\ell^\infty(\R^{d_3})\to\ell^\infty(\R^{C})}\|x^{N-1}\|_{\ell^\infty(\R^{d_3})} + \|b^{N-1}\|_{\ell^\infty(\R^{C})} \\
&\le \|x^{N-1}\|_{\ell^\infty(\R^{d_3})} + \|b^{N-1}\|_{\ell^\infty(\R^{C})}.
\end{align*}

Combining Equations \eqref{eq: thm res1 conv}-\eqref{eq: thm res1 dense} yields:
\begin{align*}
\|x^N\|_{\ell^2(\R^C)} \le \|x^0\|_{\ell^2(\R^{h_1w_1d_0})} + c(b^0,b^1,\dots,b^{N-1}),
\end{align*}
where $c(b^0,b^1,\dots,b^{N-1})$ is a constant depending on the $\ell^2$ norms of the biases in the network:
\begin{align}
c(b^0,b^1,\dots,b^{N-1}) := &\ \|b^0\|_{\ell^2(\R^{h_1w_1d_1})} + \|b^{N-1}\|_{\ell^2(\R^C)} \nonumber \\
&+ \sum_{n=1}^m\|b^n_2\|_{\ell^2(\R^{h_1w_1d_1})} + \sum_{n=m+2}^{2m}\|b^n_2\|_{\ell^2(\R^{h_2w_2d_2})} + \sum_{n=2m+2}^{3m}\|b^n_2\|_{\ell^2(\R^{h_3w_3d_3})}.  \label{eq: res1 net constant}
\end{align} 
This proves Equation \eqref{eq: res1 net}.
\end{proof}

\begin{proof}[\bf Proof of Theorem \ref{thm: res1 net}, part 2]
We will show that 
\begin{align*}
\|x^{n+1}-y^{n+1}\|_{\ell^2} \le a_n\|x^n-y^n\|_{\ell^2}
\end{align*}
for all $n=1,2,\dots,3m+3$, where $a_n\ge0$ is independent of the $x^n$ and $y^n$.

For the convolution layer (Layer 0), we have, by Equation \eqref{eq: layer conv}, that:
\begin{align}
\|x^1-y^1\|_{\ell^2(\R^{h_1w_1d_1})} &= \|A^0x^0-A^0y^0\|_{\ell^2(\R^{h_1w_1d_1})} \nonumber \\
&\le \|A^0\|_{\ell^2(\R^{h_1w_1d_0})\to\ell^2(\R^{h_1w_1d_1})}\|x^0-y^0\|_{\ell^2(\R^{h_1w_1d_0})} . \label{eq: thm res1 conv ptb}
\end{align}

For the first stack of residual layers (Layer $n$ with $n=1,2\dots,m$), we have, by Equation \eqref{eq: layer resnet1}, that:
\begin{align}
&\|x^{n+1}-y^{n+1}\|_{\ell^2(\R^{h_1w_1d_1})} \nonumber \\
&\quad= \|( x^n - A^n_2(A^n_1x^n+b^n_1)_+ + b^n_2)_+ - \ ( y^n - A^n_2(A^n_1y^n+b^n_1)_+ + b^n_2)_+\|_{\ell^2(\R^{h_1w_1d_1})} \nonumber \\
&\quad\le \|(x^n - A^n_2(A^n_1x^n+b^n_1)_+) \ - \ (y^n - A^n_2(A^n_1y^n+b^n_1)_+) \|_{\ell^2(\R^{h_1w_1d_1})} \nonumber \\
&\quad\le \|x^n - y^n\|_{\ell^2(\R^{h_1w_1d_1})}  + \|A^n_2\|_{\ell^2(\R^{h_1w_1d_1})}\|(A^n_1x^n+b^n_1)_+  -  \ (A^n_1y^n+b^n_1)_+ \|_{\ell^2(\R^{h_1w_1d_1})} \nonumber \\
&\quad\le \|x^n - y^n\|_{\ell^2(\R^{h_1w_1d_1})}  + \|A^n_2\|_{\ell^2(\R^{h_1w_1d_1})}\|A^n_1x^n- A^n_1y^n \|_{\ell^2(\R^{h_1w_1d_1})} \nonumber \\
&\quad\le \left(1+\|A^n_1\|_{\ell^2(\R^{h_1w_1d_1})}\|A^n_2\|_{\ell^2(\R^{h_1w_1d_1})}\right)\|x^n - y^n\|_{\ell^2(\R^{h_1w_1d_1})}, \label{eq: thm res1 resnet ptb}
\end{align}
where we have used the fact that ReLU is 1-Lipschitz in $\ell^2$ (see Equation \eqref{eq: prop relu2}). Analysis for the remaining residual layers, Layers $m+2$ to $2m$ and Layers $2m+2$ to $3m$, is the same.

For the first 2d pooling layer (Layer $n$ with $n=m+1$), we have, by Equation \eqref{eq: layer 2d pool}, that:
\begin{align}
&\|x^{n+1}-y^{n+1}\|_{\ell^2(\R^{h_2w_2d_2})} \nonumber \\
&\quad= \left\|\left(E(P_2(x^n)) - \left((A^n)_{|{s=2}}\,x^n+b^n\right)_+\right)_+ -\ \left(E(P_2(y^n)) - \left((A^n)_{|{s=2}}\,y^n+b^n\right)_+\right)_+ \right\|_{\ell^2(\R^{h_2w_2d_2})} \nonumber \\
&\quad\le \|E(P_2(x^n)) - E(P_2(y^n)) \|_{\ell^2(\R^{h_2w_2d_2})} + \|\left((A^n)_{|{s=2}}\,x^n+b^n\right)_+- \ \left((A^n)_{|{s=2}}\,y^n+b^n\right)_+ \|_{\ell^2(\R^{h_2w_2d_2})} \nonumber \\
&\quad\le \|x^n-y^n\|_{\ell^2(\R^{h_2w_2d_2})} + \|(A^n)_{|{s=2}}\,x^n-\ (A^n)_{|{s=2}}\,y^n \|_{\ell^2(\R^{h_2w_2d_2})} \nonumber \\
&\quad\le \left(1+\|A^n\|_{\ell^2(\R^{h_1w_1d_1})\to\ell^2(\R^{h_2w_2d_2})}\right)\|x^n - y^n\|_{\ell^2(\R^{h_1w_1d_1})}, \label{eq: thm res1 2d pool ptb}
\end{align}
where we have used the fact that padding and 2d average pooling are linear operators and are non-expansive in $\ell^2$ (see Section \ref{sec: padding pooling}). Analysis for the second 2D pooling layer, Layer $2m+1$, is the same.

For the global pooling layer (Layer $n$ with $n=3m+1$), we have, by Equation \eqref{eq: layer global pool}, that:
\begin{align}
\|x^{n+1}-y^{n+1}\|_{\ell^2(\R^{d_3})} &= \|P_g((x^n)_+)-\ P_g((y^n)_+)\|_{\ell^2(\R^{h_3w_3d_3})} \nonumber \\
&\le\|(x^n)_+ -(y^n)_+\|_{\ell^2(\R^{h_3w_3d_3})}\nonumber \\
&\le \|x^n-y^n\|_{\ell^2(\R^{h_3w_3d_3})}, \label{eq: thm res1 global pool ptd}
\end{align}
where we have used the fact that global average pooling is a linear operator and is non-expansive in $\ell^2$.

For the fully connected layer (Layer $n=3m+2=N-1$), we have, by Equation \eqref{eq: layer dense}, that
\begin{align}
\|x^{N}-y^N\|_{\ell^2(\R^{C})} &= \|W^{N-1}x^{N-1}-W^{N-1}y^{N-1}\|_{\ell^2(\R^{C})} \nonumber \\
&\le \|W^{N-1}\|_{\ell^2(\R^{d_3})\to\ell^2(\R^{C})}\|x^{N-1}-y^{N-1}\|_{\ell^2(\R^{d_3})}. \label{eq: thm res1 dense ptb}
\end{align}

Combining Equations \eqref{eq: thm res1 conv ptb}-\eqref{eq: thm res1 dense ptb} yields:
\begin{align*}
\|x^N-y^N\|_{\ell^\infty(\R^C)} \le a(A^0,A^1,\dots,W^{N-1})\, \|x^0-y^0\|_{\ell^\infty(\R^{h_1w_1d_0})}, 
\end{align*}
where $a(A^0,A^1,\dots,W^{N-1})$ is a constant depending on the $\ell^2$ norms of the filters and weights in the network:
\begin{align} 
a(A^0,A^1,\dots,W^{N-1}) := &\ \|A^0\|_{\ell^2(\R^{h_1w_1d_0})\to\ell^2(\R^{h_1w_1d_1})}  \nonumber \\
&\ \times \prod_{n=1}^m\left(1+\|A^n_1\|_{\ell^2(\R^{h_1w_1d_1})}\|A^n_2\|_{\ell^2(\R^{h_1w_1d_1})}\right) \times \left(1+\|A^{m+1}\|_{\ell^2(\R^{h_1w_1d_1})\to\ell^2(\R^{h_2w_2d_2})}\right) \nonumber \\
&\times \prod_{n=m+2}^{2m}\left(1+\|A^n_1\|_{\ell^2(\R^{h_2w_2d_2})}\|A^n_2\|_{\ell^2(\R^{h_2w_2d_2})}\right) \times \left(1+\|A^{2m+1}\|_{\ell^2(\R^{h_2w_2d_2})\to\ell^2(\R^{h_3w_3d_3})}\right) \nonumber  \\
&\times \prod_{n=2m+2}^{3m}\left(1+\|A^n_1\|_{\ell^2(\R^{h_3w_3d_3})}\|A^n_2\|_{\ell^2(\R^{h_3w_3d_3})}\right) \times \|W^{N-1}\|_{\ell^2(\R^{d_3})\to\ell^2(\R^{C})}  . \label{eq: res1 net2 constant} 
\end{align} 
This proves  Equation \eqref{eq: res1 net2}.
\end{proof}

To prove Theorem \ref{thm: res2 net}, we will first show an auxiliary result.

\begin{lemma} \label{lem: res2 net aux}
Let $A\in\R^{d\times d}$ and $b\in\R^{d}$. Define the function $F:\R^{d}\to\R^{d}$ by:
\begin{align*}
F(x):= x - A^T\sigma(Ax+b),
\end{align*}
where $\sigma$ is ReLU, {\it i.e.} $\sigma(x)=\max(x,0)$. If $\|A\|_{\ell^2(\R^{d})}\le\sqrt{2}$, then $F$ is non-expansive in $\ell^2$, {\it i.e.}
\begin{align*}
\|F(x) - F(y)\|_{\ell^2(\R^d)}\le \|x-y\|_{\ell^2(\R^d)}
\end{align*}
for all $x,y\in\R^d$.
\end{lemma}
\begin{proof}
First note that the activation function  $\sigma:\R^d\to\R^d$ is applied component-wise. The function is of bounded variation and has a derivative in the measure sense. 
Fix an index $i\in[d]$ and consider the $i$-th component $F_i$ of $F$:
\begin{align*}
F_i: \R^d\to\R, \quad F_i(x) := x_i - (A^T)_{i,:}\, \sigma(Ax+b)_i, 
\end{align*}
where $(A^T)_{i,:}$ denotes the $i$-th row of $A^T$. Its derivative $\nabla F_i$ is defined almost everywhere:
\begin{align*}
\nabla F_i: \R^d\to\R^{1\times d}, \quad \nabla F_i(x) := (e_i)^T - A^T\, \nabla \sigma(Ax+b)\, A_{i,:},
\end{align*}
where $e_i$ is the $i$-th standard basis in $\R^d$ and $A_{i,:}$ denotes the $i$-th row of $A$. For any $x,y\in\R^d$, applying the fundamental theorem of calculus yields:
\begin{align*}
F_i(x) - F_i(y) &= \int_0^1 \left((e_i)^T - A^T\, \nabla\sigma(A((1-s)y+sx)+b)\,A_{i,:}\right)(x-y)\,\dd{s} \\
&= x_i-y_i - A^T \left(\int_0^1 \nabla\sigma(A((1-s)y+sx)+b)\,\dd{s} \right)\,A_{i,:}(x-y),
\end{align*}
and thus:
\begin{align*}
F(x) - F(y) = \left(I - A^TD(x,y)A\right) (x-y),
\end{align*}
where $D(x,y)\in\R^{d\times d}$ is the diagonal matrix defined as:
\begin{align*}
D(x,y):=\int_0^1 \nabla\sigma(A((1-s)y+sx)+b)\,\dd{s}.
\end{align*}
Since ReLU is non-decreasing with derivative bounded in magnitude by 1, we have $0\le D(x,y)_{ii}\le1$ for all $i=1,2,\dots,d$. Therefore, the $\ell^2$ norm is equivalent to:
\begin{align*}
\|F(x) - F(y)\|_{\ell^2(\R^d)} = \|I - A^TD(x,y)A\|_{\ell^2(\R^d)} \, \|x-y\|_{\ell^2(\R^d)}.
\end{align*}
If $\|A\|_{\ell^2(\R^{d})}\le\sqrt{2}$, then $0 \le \lambda_{\max} (A^T D(x,y)A) \le  2$,
and thus:
\begin{align*}
 \|I - A^TD(x,y)A\|_{\ell^2(\R^d)}^2 \le 1,
\end{align*}
which implies that $F$ is non-expansive in $\ell^2$.
\end{proof}

\begin{proof}[\bf Proof of Theorem \ref{thm: res2 net}, part 1]
By replacing $\ell^\infty$ with $\ell^2$ in the proof of Theorem \ref{thm: res1 net} (part 1), one can show with a similar argument that the following bound:
\begin{align}
\|x^{n+1}\|_{\ell^2} \le \|x^n\|_{\ell^2} + c_n \label{eq: them res2 general}
\end{align}
holds for the convolution layer, the pooling layers, and the fully connected layer, where $c_n\ge0$ is independent of the $x^n$. We will show that Equation \eqref{eq: them res2 general} also hold for ResNet-S layers.

For the first stack of residual layers (Layer $n$ with $n=1,2\dots,m$), we use an alternative approach and show that:
\begin{align}
\|x^{n+1}\|_{\ell^2(\R^{h_1w_1d_1})} &\le \|x^n\|_{\ell^2(\R^{h_1w_1d_1})} + \sqrt{2}\|b^n_1\|_{\ell^2(\R^{h_1w_1d_1})} + \|b^n_2\|_{\ell^2(\R^{h_1w_1d_1})} \label{eq: thm res2 resnet}
\end{align}
provided that $\|A^n\|_{\ell^2(\R^{h_1w_1d_1})}\le\sqrt{2}$. By Equations \eqref{eq: layer resnet2} and \eqref{eq: prop relu2}, we have:
\begin{align*}
\|x^{n+1}\|_{\ell^2(\R^{h_1w_1d_1})} &\le \| x^n - (A^n)^T (A^nx^n+b^n_1)_+\|_{\ell^2(\R^{h_1w_1d_1})} + \| b^n_2\|_{\ell^2(\R^{h_1w_1d_1})}.
\end{align*}
Define $F_n:\R^{h_1w_1d_1}\to\R^{h_1w_1d_1}$ by:
\begin{align}
F_n(x):= x - (A^n)^T(A^nx+b^n_1)_+. \label{eq: layer resnet2 aux}
\end{align}
By Lemma \ref{lem: res2 net aux}, if $\|A^n\|_{\ell^2(\R^{h_1w_1d_1})}\le\sqrt{2}$, then $F_n$ is non-expansive in $\ell^2$. Therefore, 
\begin{align*}
\|x^{n+1}\|_{\ell^2(\R^{h_1w_1d_1})} &\le \|F_n(x^n)\|_{\ell^2(\R^{h_1w_1d_1})} + \|b^n_2\|_{\ell^2(\R^{h_1w_1d_1})} \\
&\le \|F_n(x^n)-F_n(0)\|_{\ell^2(\R^{h_1w_1d_1})} + \|F_n(0)\|_{\ell^2(\R^{h_1w_1d_1})} + \| b^n_2\|_{\ell^2(\R^{h_1w_1d_1})}  \\
&\le \|x^n\|_{\ell^2(\R^{h_1w_1d_1})} + \|(A^n)^T(b^n_1)_+\|_{\ell^2(\R^{h_1w_1d_1})} + \| b^n_2\|_{\ell^2(\R^{h_1w_1d_1})}  \\
&\le \|x^n\|_{\ell^2(\R^{h_1w_1d_1})} + \sqrt{2}\|(b^n_1)_+\|_{\ell^2(\R^{h_1w_1d_1})} + \| b^n_2\|_{\ell^2(\R^{h_1w_1d_1})}.
\end{align*}
Analysis for the remaining residual layers, Layers $m+2$ to $2m$ and Layers $2m+2$ to $3m$, is the same.

Combining Equations \eqref{eq: thm res1 conv}, \eqref{eq: thm res2 resnet}, and \eqref{eq: thm res1 2d pool}-\eqref{eq: thm res1 dense} yields:
\begin{align*}
\|x^N\|_{\ell^2(\R^C)} \le \|x^0\|_{\ell^2(\R^{h_1w_1d_0})} +  c(b^0,b^1,\dots,b^{N-1}), 
\end{align*}
where $c(b^0,b^1,\dots,b^{N-1})$ is a constant depending on the $\ell^2$ norms of the biases in the network:\begin{align}
c\left(\{b^n\}_{n=0}^N\right) := & \ \|b^0\|_{\ell^2(\R^{h_1w_1d_1})} + \|b^{N-1}\|_{\ell^2(\R^C)} \nonumber \\
&+\sum_{i=1}^m\left( \sqrt{2}\|b^n_1\|_{\ell^2(\R^{h_1w_1d_1})} + \|b^n_2\|_{\ell^2(\R^{h_1w_1d_1})}\right) \nonumber \\
&+ \sum_{i=m+2}^{2m}\left( \sqrt{2}\|b^n_1\|_{\ell^2(\R^{h_2w_2d_2})} + \|b^n_2\|_{\ell^2(\R^{h_2w_2d_2})}\right) \nonumber  \\
&+ \sum_{i=2m+2}^{3m}\left( \sqrt{2}\|b^n_1\|_{\ell^2(\R^{h_3w_3d_3})} + \|b^n_2\|_{\ell^2(\R^{h_3w_3d_3})}\right). \label{eq: res2 net constant}
\end{align} 
This proves Equation \eqref{eq: res2 net}.
\end{proof}

\begin{proof}[\bf Proof of Theorem \ref{thm: res2 net}, part 2]
The proof is similar to the proof of Theorem \ref{thm: res1 net} (part 2), except for the residual layers. We will show that the following bound:
\begin{align*}
\|x^{n+1}-y^{n+1}\|_{\ell^2} \le a_n\|x^n-y^n\|_{\ell^2}
\end{align*}
also holds for the residual layers, where $a_n\ge0$ is independent of the $x^n$ and $y^n$. 

For the first stack of residual layers (Layer $n$ with $n=1,2\dots,m$), we have, by Equation \eqref{eq: layer resnet2}, that:
\begin{align*}
&\|x^{n+1}-y^{n+1}\|_{\ell^2(\R^{h_1w_1d_1})} \nonumber \\
&\quad= \|( x^n - (A^n)^T(A^nx^n+b^n_1)_+ + b^n_2)_+ - \ ( y^n - (A^n)^T(A^ny^n+b^n_1)_+ + b^n_2)_+\|_{\ell^2(\R^{h_1w_1d_1})} \nonumber \\
&\quad\le \|F_n(x^n) - F_n(y^n) \|_{\ell^2(\R^{h_1w_1d_1})} ,
\end{align*}
where the function $F_n:\R^{h_1w_1d_1}\to\R^{h_1w_1d_1}$ is defined in Equation \eqref{eq: layer resnet2 aux}. By Lemma \ref{lem: res2 net aux}, $F_n$ is non-expansive in $\ell^2$ if $\|A^n\|_{\ell^2(\R^{h_1w_1d_1})}\le\sqrt{2}$. Thus,  
\begin{align}
\|x^{n+1}-y^{n+1}\|_{\ell^2(\R^{h_1w_1d_1})} \le \|x^n - y^n\|_{\ell^2(\R^{h_1w_1d_1})}, \label{eq: thm res2 resnet ptb}
\end{align}
Analysis for the remaining residual layers, Layers $m+2$ to $2m$ and Layers $2m+2$ to $3m$, is the same.

Combining Equations \eqref{eq: thm res1 conv ptb}, \eqref{eq: thm res2 resnet ptb}, and \eqref{eq: thm res1 2d pool ptb}-\eqref{eq: thm res1 dense ptb} yields:
\begin{align*}
\|x^N-y^N\|_{\ell^\infty(\R^C)} \le a(A^0,A^1,\dots,W^{N-1}) \|x^0-y^0\|_{\ell^\infty(\R^{h_1w_1d_0})}, 
\end{align*}
where $a(A^0,A^1,\dots,W^{N-1})$ is a constant depending on the $\ell^2$ norms of the filters and weights in the network:
\begin{align} 
a(A^0,A^1,\dots,W^{N-1}) := \left(1+\|A^{m+1}\|_{\ell^2(\R^{h_1w_1d_1})\to\ell^2(\R^{h_2w_2d_2})}\right) \left(1+\|A^{2m+1}\|_{\ell^2(\R^{h_2w_2d_2})\to\ell^2(\R^{h_3w_3d_3})}\right)  . \label{eq: res2 net2 constant}
\end{align} 
This proves  Equation \eqref{eq: res2 net2}. 

\end{proof}

\section{Auxiliary Results} \label{sec: aux}

To be self-contained, we include some results in differential inclusions and differential equations that we used in the main text.

\begin{definition} \label{def: prox regular}
(page 350, \cite{edmond2005sweeping}) 
For a fixed $r > 0$, the set $S$ is said to be $r$-prox-regular if, for any $x\in S$ and any $\xi\in \mathcal{N}^L_S(x)$ such that $\|\xi\| < 1$, one has $x = \proj_S(x + r\xi)$, where $\mathcal{N}^L$ denotes the limiting normal cone (see \cite{Mordukhovich1996analysis}).
\end{definition}

\begin{theorem} \label{thm: inclusion}
(Theorem 1, \cite{edmond2005sweeping}) 
Let $H$ be a Hilbert space with the associated norm $\|\cdot\|$. Assume that $C: [0,T]\to H$ with $T>0$ is a set-valued map which satisfies the following two conditions:
\begin{enumerate}[label*=(\roman*)]
\item for each $t\in [0,T]$, $C(t)$ is a nonempty closed subset of $H$ which is $r$-prox-regular;
\item there exists an absolutely continuous function $v : [0,T] \to \R$ such that for any $y\in H$ and $s,t\in [0,T]$,
\begin{align}
|\dist(y, C(t)) - \dist(y, C(s))|\le|v(t) - v(s)|. \label{eq: inclusion cond2}
\end{align}
\end{enumerate}
Let $F : [0,T] \times H \to H$ be a separately measurable map on $[0,T]$ such that
\begin{enumerate}[label*=(\roman*),start=3]
\item for every $\eta > 0$ there exists a non-negative function $k_\eta\in L^1([0,T], \R)$ such that for all $t\in [0,T]$ and for any $(x, y)\in \overline{B(0, \eta)} \times \overline{B(0, \eta)}$,
\begin{align*}
\|F (t, x) - F (t, y)\| \le k_\eta(t)\|x - y\|,
\end{align*}
where $\overline{B(0, \eta)}$ stands for the closed ball of radius $\eta$ centered at $0\in H$;
\item there exists a non-negative function $\beta\in L^1([0,T], \R)$ such that for all $t\in [0,T]$ and for all $x\in \cup_{s\in [0,T]} C(s)$, 
\begin{align*}
\|F (t, x)\| \le \beta(t)(1 + \|x\|). 
\end{align*}
\end{enumerate}
Then, for any $x_0\in C(T_0)$, where $0\le T_0<T$, the following perturbed sweeping process
\begin{align*} \begin{cases}
-\dfrac{\text{d}}{\text{dt}}x(t)\in\mathcal{N}_{C(t)}(x(t)) + F (t, x(t)) \quad \text{a.e. } t\in [0,T] \\
x(T_0) = x_0
\end{cases}
\end{align*}
has a unique absolutely continuous solution $x$, and the solution $x$ satisfies
\begin{align*}
\left\| \dfrac{\text{d}}{\text{dt}}x(t) + F (t, x(t))\right\| &\le (1 + a)\beta(t) +\left|\dfrac{\text{d}}{\text{dt}}v(t)\right| \quad \text{a.e. } t\in [0,T], \\
\|F (t, x(t))\| &\le (1 + a)\beta(t)  \quad \text{a.e. } t\in [0,T],
\end{align*}
where
\begin{align*} 
a := \|x_0\| + \exp\left(2\int^T_{T_0}\beta(s)\,\dd{s}\right)\int^T_{T_0}\left(2\beta(s)(1 +\| x_0\|) +\left|\dfrac{\text{d}}{\text{dt}}v(s)\right|\right)\,\dd{s}.
\end{align*}
\end{theorem}

\begin{remark} \label{rem: inclusion}
(Remark 2.1, \cite{kamenskii2017sweeping}) 
If $x$ is a solution to Equation \eqref{eq: sweeping process} defined on $[T_0,\infty)$, then $x(t)\in C(t)$ for all $t\in[T_0,\infty)$.
\end{remark}

The following theorem states a nonlinear generalization of Gronwall's inequality.

\begin{theorem} \label{thm: gronwall}
(Theorem 21, \cite{dragomir2003gronwall})
Let $u$ be a nonnegative function that satisfies the integral inequality
\begin{align*}
u(t) \le c + \int_{t_0}^t f(s)u(s)+g(s)u^\alpha(s) \, \dd{s},
\end{align*}
where $c\ge0$, $\alpha\ge0$, $f$ and $g$ are continuous nonnegative functions for $t \ge t_0$. 
\begin{enumerate}[label*=(\roman*)]
\item For $0\le\alpha<1$, we have:
\begin{align*}
u(t)^{1-\alpha}\le c^{1-\alpha}\exp\left((1-\alpha)\int_{t_0}^tf(s)\, \dd{s}\right) + (1-\alpha)\int_{t_0}^tg(s) \, \exp\left((1-\alpha)\int_s^tf(r)\, \dd{r}\right) \, \dd{s}.
\end{align*}
\item For $\alpha=1$, we have:
\begin{align*}
u(t)\le c\exp\left((1-\alpha)\int_{t_0}^tf(s)+g(s)\, \dd{s}\right).
\end{align*}
\end{enumerate}
\end{theorem}

\bibliographystyle{plain}
\bibliography{Paper18_ResnetStability_references}

\end{document}